\documentclass[twoside,11pt]{article}

\usepackage{blindtext}

\usepackage{subcaption}
\usepackage{thm-restate}

\usepackage{jmlr2e}

\usepackage{amsmath}
\usepackage{amsfonts}
\usepackage{algorithm}
\usepackage{algorithmic}
\usepackage{xcolor}
\usepackage{cleveref}
\usepackage{graphicx}
\usepackage{wrapfig}
\usepackage{dsfont}
\usepackage{bm}

\newcommand{\zero}{\mathbf{0}}
\newcommand{\ww}{\widehat{w^*}}
\newcommand{\av}{{a}}
\newcommand{\xv}{{x}}
\newcommand{\yv}{{y}}
\newcommand{\wv}{{w}}
\newcommand{\uv}{{u}}
\newcommand{\vv}{{v}}
\newcommand{\Uv}{{U}}
\newcommand{\Vv}{{V}}
\newcommand{\Bv}{{B}}

\newcommand{\Phiv}{{\Phi}}
\newcommand{\Sigmav}{{\Sigma}}
\newcommand{\tk}{{\tilde{k}}}
\newcommand{\tv}{\tilde{\vv}}

\newcommand{\tU}{{\widetilde{U}}}
\newcommand{\tV}{{\widetilde{V}}}
\newcommand{\ty}{\tilde{y}}
\newcommand{\tsigma}{\tilde{\sigma}}

\newcommand{\xdim}{d}
\newcommand{\nsamples}{n}
\newcommand{\s}{\sigma}

\newcommand{\Eb}{\mathbb{E}}

\newcommand{\weights}{\wv}

\newcommand{\labelvec}{\yv}
\newcommand{\uvec}{\uv}
\newcommand{\vvec}{\vv}

\newcommand{\vvecunl}{\tilde{\vv}}

\newcommand{\repmat}{\Phiv}
\newcommand{\repmatunl}{\tilde{\Phiv}}
\newcommand{\Umat}{\Uv}
\newcommand{\Smat}{\Sigmav}
\newcommand{\Vmat}{\Vv}

\newcommand{\Nc}{\mathcal{N}}
\newcommand{\zeros}{\mathbf{0}}

\newcommand{\eye}{\mathbb{I}}

\newcommand{\Rb}{\mathbb{R}}
\newcommand{\Sc}{\mathcal{S}}
\newcommand{\Tc}{\mathcal{T}}

\newcommand{\norm}[1]{\left\lVert #1 \right\rVert}
\newcommand{\rank}[1]{r(#1)}

\usepackage{lastpage}
\jmlrheading{25}{2024}{1-\pageref{LastPage}}{7/23}{8/24}{23-0899}{Ehsan Imani, Guojun Zhang, 
	Runjia Li, Jun Luo, Pascal Poupart, Philip H.S. Torr, and Yangchen Pan}

\ShortHeadings{Label Alignment Regularization for Distribution Shift}{Imani, Zhang, Li, Luo, Poupart, Torr, and Pan}
\firstpageno{1}

\begin{document}
	
	\title{Label Alignment Regularization for Distribution Shift}
	
	\author{\name Ehsan Imani\thanks{Work done during employment at Huawei Noah's Ark Lab.} \email imani@ualberta.ca \\
		\addr University of Alberta, Alberta Machine Intelligence Institute
		\AND
		\name Guojun Zhang \email guojun.zhang@huawei.com \\
		\addr Huawei Noah's Ark Lab
		\AND
		\name Runjia Li \email runjia@robots.ox.ac.uk \\
		\addr Department of Engineering Science, University of Oxford
		\AND
		\name Jun Luo \email jun.luo1@huawei.com \\
		\addr Huawei Noah's Ark Lab
		\AND
		\name Pascal Poupart \email ppoupart@uwaterloo.ca \\
		\addr School of Computer Science, University of Waterloo
		\AND
		\name Philip H.S. Torr \email philip.torr@eng.ox.ac.uk \\
		\name Yangchen Pan\footnotemark[1] \email yangchen.pan@eng.ox.ac.uk \\
		\addr Department of Engineering Science, University of Oxford}
	
	\editor{Amos Storkey}

\maketitle

\begin{abstract}
	Recent work has highlighted the label alignment property (LAP) in supervised learning, where the vector of all labels in the dataset is mostly in the span of the top few singular vectors of the data matrix. Drawing inspiration from this observation, we propose a regularization method for unsupervised domain adaptation that encourages alignment between the predictions in the target domain and its top singular vectors. Unlike conventional domain adaptation approaches that focus on regularizing representations, we instead regularize the classifier to align with the unsupervised target data, guided by the LAP in both the source and target domains. Theoretical analysis demonstrates that, under certain assumptions, our solution resides within the span of the top right singular vectors of the target domain data and aligns with the optimal solution. By removing the reliance on the commonly used optimal joint risk assumption found in classic domain adaptation theory, we showcase the effectiveness of our method on addressing problems where traditional domain adaptation methods often fall short due to high joint error. Additionally, we report improved performance over domain adaptation baselines in well-known tasks such as MNIST-USPS domain adaptation and cross-lingual sentiment analysis. An implementation is available at \url{https://github.com/EhsanEI/lar/}.
\end{abstract}


\section{Introduction}
\label{sec:intro}

Unsupervised domain adaptation studies knowledge transfer from a source domain with labeled data,
to a target domain with unlabeled data, where the model will be deployed and evaluated \citep{ben2010theory,mansour2009domain}. This difference between the two domains, called domain shift, arises in many applications. A document classification or sentiment analysis model for an under-resourced language can benefit from a large corpus for a different language. A personal healthcare system is often trained on a group of users different from its target users. A real-world robot's predictions or decision-making can improve through safe and less costly interactions with a simulator \citep{pires2019multilingual,ganin2016domain,peng2018sim}.

There are diverse settings to study domain adaptation problems. In classification problems, closed set domain adaptation assumes the same categories between the two domains while open-set domain adaptation assumes that the two domains only share a subset of their categories \citep{panareda2017open}. Unsupervised, semi-supervised, and supervised domain adaptation assume that the data from the target domain is fully unlabeled, partly labeled, and fully labeled respectively \citep{ganin2016domain}. Two related problems to domain adaptation are multi-target domain adaptation where there are multiple target domains \citep{gholami2020unsupervised} and domain generalization where several source domains are sampled from a distribution over tasks and the goal is to generalize to a previously unseen domain from this distribution \citep{blanchard2011generalizing,gulrajani2020search}. Within those diverse settings, our work specifically addresses unsupervised domain shift problems.

The prevalence of domain shift in machine learning has inspired a large body of algorithmic and theoretical research on domain adaptation. \citet{ben2010theory} and \citet{zhang2019bridging} formulated the difference between the source and the target domain with the notion of $\mathcal{H}$-divergence and Margin Disparity Discrepancy and provided generalization bounds that relate performance on the two domains. \citet{acuna2021f} extended these results to a more general notion of $f$-divergence. Adversarial domain adaptation algorithms are motivated by these theoretical findings and aim to learn representations that achieve high performance in the source domain while being invariant to the shift between the source and the target domain \citep{ganin2016domain, zhang2017aspect,conneau2017word,long2015learning,pei2018multi}. 

The aforementioned representation-matching approach assumes that the optimal joint risk between the source and target is small. 
This assumption fails when the conditional distribution of the labels given input is different between source and target domains.
An example occurs when labels in the source domain are much more imbalanced than in the target domain. For instance, \citet{zhao2019learning} identified that under such label distribution shift, the optimal joint risk can be quite large and they empirically show the failure of domain adaptation methods on MNIST-USPS digit datasets. 
\citet{johansson2019support} also pointed out the limitation of matching feature representations by showing its inconsistency, and thus the tendency for high target errors. 

In this work, we adopt a novel approach to domain adaptation that focuses on label alignment, defined as the alignment of labels with the top left singular vectors of the representation. Instead of striving for an invariant representation, our proposed algorithm fine-tunes the classifier for the target domain. It achieves this by removing the influence of label alignment in the source domain and applying this alignment principle to the target domain. A critical distinction of our approach from existing methodologies is that we adjust the classifier's weight rather than its representation. Consequently, our method can be applied in settings with linear function approximation and may complement existing approaches. 

We describe the label alignment phenomenon in Section \ref{sec:background}, and outline the proposed method in Section \ref{sec:algorithm}. Section \ref{sec:theory} formally justifies our regularization method by showing that it projects the solution onto the span of the top right singular vectors of the target domain. Section \ref{sec:related} reviews related work. In Section \ref{sec:experiments}, we first provide a synthetic example where the proposed regularizer shows a clear advantage. We then experiment with imbalanced MNIST-USPS binary classification tasks and find that our method, unlike the domain-adversarial baseline, is robust to imbalance in one domain. Finally, we evaluate our algorithm on cross-lingual sentiment analysis tasks and observe improved $F_1$ score on training with our regularization, compared to adversarial domain adaptation baselines. 

\section{Background: Label Alignment}
\label{sec:background}

In this section, we briefly review the standard linear regression problem and define relevant notations to explain the \emph{label alignment property}~\citep[LAP,][]{imani2022representation}. 
\subsection{Linear Regression and Notations}\label{sec:bg-lr}
We consider a dataset with $\nsamples$ samples, (possibly learned and nonlinear) representation matrix $\repmat \in \Rb^{\nsamples\times\xdim}$ and label vector $\yv \in \Rb^n$ from a source domain. Denote the model's weights as $\weights \in \Rb^{\xdim}$, we study the linear regression problem:
\begin{align}\label{eq:linear_regression}
\min_w \|\Phiv \wv - \yv \|^2.    
\end{align}
Without loss of generality, we replace the bias unit with a constant feature in the representation matrix to avoid studying the unit separately. 
The model will be evaluated on a test set sampled from the target domain. 

The singular value decomposition (SVD) of a representation matrix $\repmat$ is $\repmat = \Umat\Smat\Vmat^\top = \sum_{i=1}^d \sigma_i \uv_i \vv_i^\top$, where 
\[
  \Smat =
  \begin{bmatrix}
    \sigma_{1} & & \\
    & \ddots & \\
    & & \sigma_{d} \\
    & {\bf 0} & 
  \end{bmatrix} \in \Rb^{\nsamples\times\xdim}
\]
is a rectangular diagonal matrix whose main diagonal consists of singular values $\s_1, \cdots, \s_\xdim$ in descending order with the remaining rows set to zero, and $$\Umat = [\uvec_1, \dots, \uvec_\nsamples ] \in \Rb^{\nsamples\times\nsamples} \mbox{ and }\Vmat = [\vvec_1, \cdots, \vvec_\xdim] \in \Rb^{\xdim\times\xdim}$$ are orthogonal matrices whose columns $\uv_i \in \Rb^n$ and $ \vvec_j \in \Rb^d$ are the corresponding left and right singular vectors. In principal component analysis \citep{pearson1901liii}, $\vv_1, \cdots, \vv_k$ are also known as the first $k$ principal components.
For a vector $\av$ and orthonormal basis $\Bv$, $\av^{\Bv}$ is a shorthand for $\Bv^\top \av$, the representation of $\av$ in terms of the row vectors of $\Bv$. We use $\rank{\cdot}$ to denote the rank of a matrix.

\subsection{Label Alignment}
\label{sec:bg-alignment}



Label alignment is specified in terms of the singular vectors of $\repmat$ and label vector $\labelvec$. The left singular vectors of $\repmat$, $\{\uv_1, \cdots, \uv_n\}$ form an orthonormal basis that spans the $\nsamples$-dimensional space. The label vector $\yv \in \Rb^n$ can be decomposed in this basis with:
\begin{align}
\yv = \Uv \yv^{\Uv} = y_1^{\Uv} \uv_1 + \cdots + y_n^{\Uv} \uv_n,
\end{align}
where $y^\Umat_i$ is the $i^{\rm th}$ component of vector $\yv^\Umat \in \mathbb{R}^n$.

Label alignment \citep{imani2022representation} is a relationship between the labels and the representation where the variation in the labels are mostly along the top principal components of the representation. For our purpose we give the following definition and verify that it approximately holds in a number of real-world tasks. A dataset has \emph{label alignment} with rank $k$ if for $k \ll \rank{\Phi}$ we have  $y_i^{\Uv} = 0, \forall i\in\{k+1,...,d\}$.

In Table \ref{tab:alignment_tasks} we investigate this property in binary classification tasks (with $\pm1$ labels) and regression tasks. In this table $k(\epsilon)$ means the smallest $k$ where 
\[
\sqrt{\sum_{i=k+1}^d (y_i^{\Uv})^2} < \epsilon \sqrt{\sum_{i=1}^d (y_i^{\Uv})^2}.
\]
If $k(\epsilon)$ is small for a small $\epsilon$ then the projection of the label vector on the span of $\Phiv$ is mostly in the span of the first few singular vectors. In all the ten tasks less than half the singular vectors with nonzero singular values already span $\ge 90\%$ of the norm of the projection of $\labelvec$ on the span of $\Phiv$. The number $k(0.1)$ is remarkably small, less than $10$, in seven out of the ten tasks. Appendix~\ref{app:labelalign} shows this property in a controlled setting where a large number of features are correlated with the labels.





\begin{table}[ht]
    \centering
\begin{tabular}{lllll}
\hline
 Task     &     $n$ &   $d$ &   $\rank{\Phi}$ &   $k(0.1)$ \\
\hline
 CT Scan   & 10000 & 385 &          372 &       12 \\
 Song Year & 10000 &  91 &           91 &        6 \\
 Bike Sharing     & 10000 &  13 &           13 &        4 \\
 MNIST    & 12665 & 785 &          580 &        2 \\
 USPS     &  2199 & 257 &          257 &        2 \\
\hline
\end{tabular}
\begin{tabular}{lllll}
\hline
 Task     &     $n$ &   $d$ &   $\rank{\Phi}$ &   $k(0.1)$ \\
\hline
 CIFAR-10  & 10000 & 513 &          513 &        7 \\
 CIFAR-100 &  1000 & 513 &          513 &        7 \\
 STL-10    &  1000 & 513 &          513 &        2 \\
 XED (En)   &  6525 & 769 &          769 &      231 \\
 AG News   & 10000 & 769 &          769 &       40 \\
\hline
\end{tabular}
\small
    \caption{Label alignment in real-world tasks. The table on the left uses the original features in the dataset and the table on the right uses features extracted from neural networks. CT Scan, Song Year, and Bike Sharing are regression tasks and the rest are binary classification. We used the first two classes of multi-class classification datasets to create a binary classification task. Other details about the datasets are in Appendix B. In all of these tasks, a large portion of the label vector is in the span of a relatively small set of top singular vectors (compared to the rank).}
    \label{tab:alignment_tasks}
    \vspace{-0.5cm}
\end{table}

Similar label alignment phenomenon has been also observed in a deep learning setting. Recent work in the Neural Tangent Kernel (NTK) literature has observed that in common datasets the label vector is largely within the span of the top eigenvectors of the NTK Gram matrix \citep{arora2019fine}. In contrast, a randomized label vector would be more or less uniformly aligned with all eigenvectors. 
More recently, \citet{baratin2021implicit} and \citet{ortiz2021can} noted that training a finite-width NN makes the alignment between the network's kernel and the task even stronger. \citet{imani2022representation} observed a similar behavior in NN hidden representations, indicating that training the NN aligns the top singular vectors of the hidden representations to the task.

\subsection{Reformulating the Regression Objective}\label{sec:bg-rewritelr}

We describe how to reformulate the linear regression objective function with the label alignment property. This reformulation shows that the linear regression objective is implicitly enforcing the LAP on the source domain (i.e., the training data) and this encourages us to further derive our domain adaptation regularization on the target domain.

Objective \eqref{eq:linear_regression} can be rewritten by the following steps. 

\begin{align}\label{eq:lr-rewrite}
\min_{\wv} \|\Phiv \wv - \yv\|^2 &= \min_{\wv} \| \Uv\Sigmav \Vv^\top \wv - \yv\|^2 \nonumber \\
&= \min_{\wv} \|\Sigmav \Vv^\top \wv - \Uv^\top \yv\|^2 \nonumber \\
&= \min_{\wv} \|\Sigmav \wv^{\Vv} - \yv^{\Uv}\|^2 \nonumber, \text{shorthand notation} \\
&= \min_\wv \sum_{i=1}^d (\sigma_i w_i^{\Vv} - y_i^{\Uv})^2 + \sum_{i={d+1}}^n (y_i^{\Uv})^2.
\end{align}

In the first line, since $\Uv$ is an orthogonal matrix, we have $\Uv\Uv^\top = \eye$ and $\|\Uv \xv\| = \|\xv\|$ for any vector $x$. Note that the last term $\sum_{i={d+1}}^n (y_i^{\Uv})^2$ can be dropped as it is a constant and does not affect the optimization. 

Assume the LAP holds for the first $k<d$ singular vectors. Then $y_{i}^{\Uv} = 0, \forall i \in {k+1,...,d}$. Hence the first term in~\eqref{eq:lr-rewrite} can be further decomposed to
\begin{align*}
\sum_{i=1}^d (\sigma_i w_i^{\Vv} - y_i^{\Uv})^2 = 
\sum_{i=1}^k (\sigma_i w_i^{\Vv} - y_i^{\Uv})^2 + \sum_{i={k+1}}^d \sigma_i^2 (w_i^{\Vv})^2.
\end{align*}
Plugging this decomposition into the above objective~\eqref{eq:lr-rewrite}  and dropping the last term, we get
\begin{align}\label{eq:lr-rewrite-final}
\min_{w} \sum_{i=1}^k (\sigma_i w_i^{\Vv} - y_i^{\Uv})^2 + \sum_{i={k+1}}^d \sigma_i^2 (w_i^{\Vv})^2.
\end{align}
 We can interpret the first term in the rewritten objective~\eqref{eq:lr-rewrite-final} as linear regression on a smaller subspace and the second term as a regularization term implicitly enforcing label alignment property on the training data $(\Phiv, \yv)$. 

The latter is because minimizing the second term has the effect of regularizing the predictions so they likely align with the top singular vectors. This is because: 
\[
\labelvec = \repmat \weights = \Umat \Smat \Vmat^\top \weights
\]
and therefore $\Uv^\top \labelvec = \Smat \Vmat^\top \weights$, which can be written as 
\[
\yv^{\Uv} = \Smat \weights^\Vmat
\]
by using the shorthand notations. For the $i$th component in vector $\yv^{\Uv}$, we have ${\uvec_i}^\top \labelvec = \s_i {\vvec_i}^\top\weights$. Minimizing $w^{\Vmat}_i$ for $i\in\{k+1,\cdots,d\}$ will reduce the corresponding $y^{\Umat}_i$ and leave $y^{\Umat}_i$ for those components $i<k+1$. We call the second term $\sum_{i={k+1}}^d \sigma_i^2 (w_i^{\Vv})^2$ from~\eqref{eq:lr-rewrite-final} \emph{label alignment regularization}. 

The derivation above shows that when minimizing the original mean squared error for linear regression, we implicitly use label alignment regularization on the training data (source domain data). In the next section, we introduce this regularization into the target domain. 


\section{Label Alignment for Domain Adaptation}
\label{sec:algorithm}

This section describes our approach to domain adaptation by enforcing the LAP. 

In unsupervised domain adaptation, we have a labeled dataset $(\Phiv, \yv)$ and an unlabeled dataset $\tilde{\Phiv}$ with the corresponding label vector $\tilde{\yv}$ unknown. From~\eqref{eq:lr-rewrite-final}, we know that enforcing the LAP does not require knowing the labels $\tilde{\yv}$. This inspires our key idea of improving the generalization on the target domain: we can use the unlabeled part to enforce the LAP. 

Using tilde notation for the SVD of $\tilde{\Phiv}$ and assuming $(\tilde{\Phiv},\tilde{\yv})$ satisfies the LAP with rank $\tilde{k}$, we can put together the supervised part of the source domain and unsupervised part of the target domain to form the objective:
\begin{align}\label{eq:semi_supervised_loss}
\min_\wv \|\Phiv \wv - \yv \|^2 \!+ \! \sum_{i=\tilde{k}+1}^d \tilde{\sigma}_i^2 (w_i^{\tilde{\Vv}})^2.
\end{align}
The second term $\sum_{i=\tilde{k}+1}^d \tilde{\sigma}_i^2 (w_i^{\tilde{\Vv}})^2$ is the \emph{label alignment regularization on the target domain}. As we explained in the previous section, the first term (i.e. the standard regression part) in the above objective implicitly enforces the LAP (with rank $k$) on the source domain. If we expand \eqref{eq:semi_supervised_loss} by the reformulated linear regression objective~\eqref{eq:lr-rewrite-final}, we have:
\begin{align*}
\min_{\wv} \sum_{i=1}^k (\sigma_i w_i^{\Vv} - y_i^{\Uv})^2 \!+\! \sum_{i={k+1}}^d \sigma_i^2 (w_i^{\Vv})^2 \!+ \!\sum_{i=\tilde{k}+1}^d \tilde{\sigma}_i^2 (w_i^{\tilde{\Vv}})^2.
\end{align*}
Therefore, we have actually done the regularization \emph{twice}: one with the source domain and one with the target domain. We explicitly remove the label alignment regularization on the source domain and arrive at the final objective function:
\begin{small}
\begin{align}\label{eq:regression_reduction}
\min_{\wv} \|\Phiv \wv - \yv \|^2 - \sum_{i={k+1}}^d \sigma_i^2 (w_i^{\Vv})^2 + \lambda \sum_{i=\tilde{k}+1}^d \tilde{\sigma}_i^2 (w_i^{\tilde{\Vv}})^2.
\end{align}
\end{small}
Algorithm \ref{alg1} shows the pseudo-code. The objective to be minimized has three terms and the hyperparameter $\lambda$ controls the relative importance of the regularizer. As we will show in \S~\ref{sec:theory}, under certain constraints this hyperparameter does not affect the final solution and only changes the convergence rate. The first term is the loss that uses the labeled data from the source domain. Following the recent evidence on the viability of the squared error loss for classification \citep{hui2020evaluation}, we use the squared error in both regression tasks and binary classification tasks. We use $\pm1$ labels in binary classification as these labels showed the label alignment property (LAP) in Table \ref{tab:alignment_tasks}. The second term removes implicit regularization from the source domain. The third term is the proposed regularizer that uses the unlabeled data from the target domain. The second and third terms serve as a projection onto the orthogonal complement of ${\rm span}(\tv_{k+1}, \dots, \tv_{\xdim})$, or namely, ${\rm span}(\tv_1, \dots, \tv_k)$, which we show in the next section. 

\begin{algorithm}
\caption{Label Alignment Regression}\label{alg1}
\begin{algorithmic} 
\STATE Get data $\repmat$, $\labelvec$, $\repmatunl$, and hyperparameters $t$, $\alpha$, $k$, $\tilde{k}$, $\lambda$
\STATE Compute covariance matrices $\repmat^\top\repmat$ and $\repmatunl^\top\repmatunl$
\STATE Perform eigendecomposition of $\repmat^\top\repmat$ and $\repmatunl^\top\repmatunl$ to get $\sigma_{k+1:\xdim}$, $\tilde{\sigma}_{\tilde{k}+1:\xdim}$, $\vvecunl_{k+1:\xdim}$ and $\vvecunl_{\tilde{k}+1:\xdim}$
\STATE Initialize $\weights$ to zero
\FOR {$t$ iterations}
	\STATE Perform gradient step with respect to $\|\Phiv \wv - \yv \|^2 - \sum_{i={k+1}}^d \sigma_i^2 (w_i^{\Vv})^2 + \lambda \sum_{i=\tilde{k}+1}^d \tilde{\sigma}_i^2 (w_i^{\tilde{\Vv}})^2$ with step-size $\alpha$ and update $\weights$
\ENDFOR
\end{algorithmic}
\end{algorithm}



\newcommand{\Mc}{\mathcal{M}}
\section{Label Alignment Regularization as Projection}
\label{sec:theory}

In this section, we provide theoretical insight into how close the solution acquired by our regularization approach is to the optimal solution on the target domain. First, we use a simple rotated Gaussian example to illustrate that our label alignment can exactly give the optimal target solution (see also \S~\ref{sec:experiments}). Second, we generalize our conclusion beyond the Gaussian example and present the main theorem, showing that when $k=\tilde{k}$ and under a weak additional assumption our solution 1) lies in the span of the top few singular vectors of the target domain and 2) lies in the same direction as the optimal target domain solution under certain assumptions. All proofs in this section are in Appendix~\ref{app:proofs-solutions}. 

For convenience, we rewrite our objective \eqref{eq:regression_reduction} as:
\begin{align*}
\min_w \|\Phi w - y \|^2 - w^\top (S - S_k) w + \lambda w^\top (\tilde{S} - \tilde{S}_{\tilde{k}}) w,
\end{align*}
where $S = {\Phi}^\top {\Phi}$ is the covariance matrix of ${\Phi}$, $S_{{k}}$ is the covariance matrix truncated to rank ${k}$ and similar notations hold for $\tilde{S}$ and $\tilde{S}_{\tk}$. Then the optimal solution for this problem is:
\begin{align}\label{eq:learned_solution}
\widehat{w^*} = (S_k + \lambda(\tilde{S} - \tilde{S}_{\tilde{k}}))^{-1}\Phi^\top y,
\end{align}
if the matrix $S_k + \lambda(\tilde{S} - \tilde{S}_{\tilde{k}})$ is full rank, which requires $k \geq \tk$. In practice, we can treat $k$ and $\tk$ as hyper-parameters and choose them as we wish.

\subsection{Rotated Gaussian Example}\label{sec:RGaussian}

Consider a simple example where the source and target domain data are both two-dimensional Gaussians, but the target domain is acquired by rotating the source domain (Figure 1 provides a concrete example). Denote the following Gaussian distribution as:
\begin{align}
\Nc(0, Q) = \frac{1}{2\pi \sqrt{|Q|}} \exp\left( -\frac12 x^\top Q^{-1} x\right),
\end{align}
where
$
Q = P \begin{bmatrix} s_1^2 & 0 \\
0 & s_2^2 \end{bmatrix} P^\top, \,\mbox{ and } P = \begin{bmatrix}
p_1 & p_2
\end{bmatrix}$.  
Here we consider the spectral decomposition of the covariance matrix $Q \in \Rb^{2\times 2}$ with $s_1 > 0$, $s_2 > 0$.
Here $P \in \Rb^{2\times 2}$ is an orthogonal matrix, and $p_1$, $p_2$ are its column vectors. Since $x = P P^\top x = x_1^P p_1 + x_2^P p_2$, we can rewrite the distribution as: 
\begin{align}
\Nc(0, Q) = \frac{1}{2\pi s_1 s_2}\exp\left( -\frac{1}{2 s_1^2}(x_1^P)^2 -\frac{1}{2 s_2^2}(x_2^P)^2 \right). \nonumber
\end{align}
We further define the conditional distributions as follows:
\begin{align}\label{eq:def-pxgiveny}
p_{\Sc}(x | y) &= 2 \Nc(0, Q) \mathds{1}(y x_1^P > 0),
\end{align}
where $y\in \{1, -1\}$. Similarly, we can define the target distribution by replacing $Q, P, s_i, p_i$ with $\tilde{Q}, \tilde{P}, \tilde{s}_i, \tilde{p}_i$. We now compute different solutions and then compare them. We assume that there is distribution shift and that $p_1$ is not parallel to $\tilde{p}_2$. 

Recall the regression solution on the source domain:
\begin{align}\label{eq:source_sol}
w_\Sc^* = (\Phi^\top \Phi)^{-1} \Phi^\top y = S^{-1} \Phi^\top y.
\end{align}
Assuming that the sample size is large enough. 
\begin{align}
\frac{1}{n}\Phi^\top y &\approx \Eb_{x, y} [x y] = \sqrt{\frac{2}{\pi}} s_1 p_1,
\end{align}
where the computation of the expectation is detailed in Lemma~\ref{lem:expect-xy}, Appendix \ref{app:proofs-solutions}. Combining this equation with
\begin{small}
\begin{align}\label{eq:Phi_y}
\Phi^\top y = V \Sigma^\top y^U = \sigma_1 y_1^U v_1 + \sigma_2 y_2^U v_2,
\end{align}
\end{small}
we know that $y_2^U = 0$ if we identify $v_1 = p_1$. In other words, the label alignment property holds on the source domain with rank $k = 1$. The covariance matrix is:
\begin{align}\label{eq:Phi_Phi}
\frac1n \Phi^\top \Phi &\approx \Eb_x [xx^\top] = s_1^2 p_1 p_1^\top + s_2^2 p_2 p_2^\top,
\end{align}
We can identify $v_i = p_i$, $s_i^2 = \sigma_i^2/n$ from the SVD of $\Phi$.
Plugging \eqref{eq:Phi_y} and \eqref{eq:Phi_Phi} back into \eqref{eq:source_sol} we get the optimal solution on the source domain:
\begin{small}
\begin{align}
w_\Sc^* = \sqrt{\frac{2}{\pi}} \frac{1}{s_1} v_1,
\end{align}
\end{small}
which agrees with our intuition that $w_\Sc^*$ should be in the direction with the largest singular value. Similarly, the optimal solution on the target domain is
\newcommand{\ts}{\tilde{s}}
\begin{small}
\begin{align}
w_\Tc^* = \sqrt{\frac{2}{\pi}} \frac{1}{\ts_1} \tilde{v}_1,
\end{align}
\end{small}
where the tilde notations are the same type of variables used on the target domain. 
According to \eqref{eq:learned_solution}, the label alignment solution with the removal of implicit regularization (given that $\tv_2$ is not parallel to $v_1$) is:
\begin{align}\label{eq:label_alignment}
\widehat{w^*} &= (S_k + \lambda(\tilde{S} - \tilde{S}_{\tilde{k}}))^{-1}\Phi^\top y \\
&= (s_1^2 v_1 v_1^\top + \lambda\tilde{s}_2^2 \tilde{v}_2 \tilde{v}_2^\top)^{-1} \sqrt{\frac{2}{\pi}} {s_1} v_1,
\end{align} 
To better understand the solution $\ww$, suppose $\lambda = 1$. Then if we replace $\tilde{s}_2$, $\tv_2$ by $s_2$, $v_2$, we obtain $w_\Sc^*$. If we replace $s_1$, $v_1$ by $\ts_1$, $\tv_1$, we obtain $w_\Tc^*$.

In fact in this example the effect of label alignment regularization is some kind of projection into the space of $\tilde{v}_1$. Regardless of the hyperparameter $\lambda$, we always have the following result:
\begin{restatable}[]{proposition}{rGauss}\label{prop-gaussian}
In the example in this section suppose $v_1^\top \tilde{v}_1 \neq 0$. Then the label alignment solution is $\widehat{w^*} = c w_\Tc^* / v_1^\top \tilde{v}_1$ with $c > 0$.
\end{restatable}
The proposition shows that given $v_1^\top \tilde{v}_1 > 0$, our solution $\widehat{w^*}$ is exactly in the same direction as the optimal solution $w_\Tc^*$, which is verified in our experiments (\Cref{fig:rotated_Gaussian}). The above discussion also holds in a more generalized setting, as we show below.


\subsection{Generalized Setting}

This section derives the relation between the solutions $\widehat{w^*}$ and $w_\Tc^*$ in a more general setting, where $x$ is high dimensional, and $k, \tilde{k}$ can be larger than one. We can rewrite
\begin{align}
w_\Sc^* = \sum_{i\leq k} \sigma_i^{-1}y_i^U v_i,\, w_\Tc^* = \sum_{i\leq \tilde{k}} \tsigma_i^{-1}\ty_i^\tU \tv_i.
\end{align}
Hence, $w_\Sc^* \in {\rm span}(v_1, \dots, v_k)$, $w_\Tc^* \in {\rm span}(\tv_1, \dots, \tv_k)$. We show that our solution is also in the span of the top right singular vectors of the target domain as $w_\Tc^*$: 
\begin{restatable}[]{theorem}{alignment}\label{thm:alignment}
Assume $k = \tk$ and $(V'_{d-k})^\top \tilde{V}_{d-k}'$ is invertible with
$
V'_{d-k} = \begin{bmatrix}v_{k + 1} & \dots& v_d \end{bmatrix}$ and $ \tilde{V}_{d-k}' = \begin{bmatrix}\tv_{\tk + 1} & \dots& \tv_d \end{bmatrix}, 
$
then $\widehat{w^*} \in {\rm span}(\tv_1, \dots, \tv_k)$ holds and $\ww$ is independent of $\lambda$.
\end{restatable}
This theorem tells us that after label alignment regularization, $\widehat{w^*}$ and $w_\Tc^*$ lie in the same subspace. 

We now characterize when our solution can lie in exactly the same direction as the optimal target domain's solution. Denote $V_k = [v_1 \dots v_k]$,\, $\tV_{\tk} = [\tv_1  \dots \tv_{\tk}]$, and 
\begin{align}
\mu_k = (y_1^U/\sigma_1, \dots, y_k^U/\sigma_k), \, \tilde{\mu}_{\tk} = (\ty_1^\tU/\tsigma_1, \dots, \ty_\tk^\tU/\tsigma_\tk). \nonumber
\end{align}
We have the following theorem:
\begin{restatable}[]{theorem}{optimal}\label{thm:proportional}
Given invertible $V_k^\top \tV_k$, with $V_k = [v_1 \dots v_k]$,\, $\tV_{\tk} = [\tv_1  \dots \tv_{\tk}]$ and with the same assumptions of \Cref{thm:alignment}, there exists $c> 0$ such that $\ww = c w_\Tc^*$ iff:
\begin{small}
\begin{align}\label{eq:condition_mu}
\mu_k = c V_k^\top \tV_k\tilde{\mu}_k = c V_k^\top w_\Tc^*.
\end{align}
\end{small}
\end{restatable}

In the special case of $k = \tk = 1$, we obtain the following: 
\begin{restatable}[]{corollary}{rankone}\label{cor-highdim-proportional}
Given $k = \tk = 1$ and $\tilde{y}^{\tU}_1 y_1^U v_1^\top \tilde{v}_1 > 0$, we have $\ww = c w_\Tc^*$.
\end{restatable}
This corollary tells us that if $k = \tk = 1$ and if, for both domains, the labels can be determined by the principal component (or, in other words, the most significant feature), then our label alignment regularization finds the optimal target solution. 

Back in the more general setting of $k = \tk \geq 1$, we show a sufficient condition for the invertibility assumption in Theorem \ref{thm:alignment} to hold: $V$ and $\tV$ are somehow similar to each other. 
\begin{restatable}[]{proposition}{invertible}\label{prop-inv-assump}
Suppose $\epsilon < \min\{\frac1k, \tfrac{1}{d-k}\}$, $|v_i^\top \tv_j| \leq \epsilon$ for any $i\neq j$, and $v_i^\top \tv_i \geq 1 - \epsilon$ for any $i$, then both $V_k^\top \tV_k$ and $(V'_{d-k})^\top \tilde{V}_{d-k}'$ are invertible.
\end{restatable}

We can give a stronger guarantee for the assumption that $V_k^\top \tilde{V}_k$ is invertible. Note that 
$\mathds{S}^{d-1}$ denotes the $(d-1)$-dimensional unit hypersphere in $\Rb^d$. 
\begin{restatable}[]{proposition}{invertibleStrong}\label{prop:inv-assump-strong}
Suppose the target singular vectors $\tv_1, \dots, \tv_d$ satisfies the following probability distribution:
\begin{align}
p(\tv_1, \dots, \tv_d) = p(\tv_1) p(\tv_2 | \tv_1) \dots p(\tv_d | \tv_1, \dots, \tv_{d-1}),  
\end{align}
where $p(\tv_1)$ is a continuous distribution on $\mathds{S}^{d-1}$ and each $p(\tv_i | \tv_1, \dots, \tv_{i-1})$ is a continuous distribution on the manifold $\mathds{S}^{d-1}$ for $2\leq i \leq d$. Then $V_k^\top \tV_k$ is invertible almost surely. 
\end{restatable}

Note that our result does not depend on any assumption about the optimal joint error, as is commonly required in the domain adaptation literature \citep[e.g.][]{ben2010theory, acuna2021f}. Moreover, as pointed out by \citet{zhao2019learning}, the usual generalization bound would fail in the presence of heavy shift of label distributions, under which our method is still robust (see \Cref{sec:experiments}).


\section{Related Work}
\label{sec:related}

The result by \citet{ben2010theory} provides a general theoretical guidance regarding how to learn the domain-invariant representations. The basic idea is to make the joint error of the best hypothesis on the two domains on the invariant representation small. Low joint error in the domain-adversarial model is crucial to the model's performance on the target domain.

The dominant approach to domain adaptation is learning domain-invariant representations that are similar in some sense between source and target domains \citep{tzeng2014deep,zhuang2015supervised,ghifary2016deep,long2016unsupervised,long2017deep,benaim2017one,bousmalis2017unsupervised,courty2017joint,motiian2017few,rebuffi2017learning,saito2017asymmetric,zhang2019bridging}. Different methods differ in how the invariance property is enforced, which typically includes how the similarity is defined and implemented. Recent work in deep learning encourages this invariance in one or multiple hidden representations of a neural network.

The popular domain-adversarial methods achieve domain-invariant representations based on the idea of adversarial models~\citep{long2015learning,zhuang2017supervised,lee2019sliced,damodaran2018deepjdot,acuna2021f}. Specifically, \citet{long2015learning} adversarially learn representations to distinguish the data points from the source and target domain while minimizing the supervised loss. \citet{conneau2017word} use a domain-adversarial approach to align representations of the source and target domains in a shared space. They transform the source embeddings with a linear mapping that is encouraged to be orthogonal. The domain-adversarial model then generates pseudo-labels on the target domain for additional refinement. The shared representation, which is learned without a parallel corpus, outperforms previous supervised methods in several cross-lingual tasks.

There are various similarity or distance measures to define a loss function for enforcing invariant representations. For example, \citet{zhuang2017supervised} and \citet{meng2018adversarial} minimize the KL-divergence and \citet{lee2019sliced} and \citet{damodaran2018deepjdot} minimize the Wasserstein distance. \citet{sun2016deep} minimize the $\ell_2$ distance between the covariance matrices of the source and target domain representations. \citet{long2015learning} minimize Maximum Mean Discrepancy between source and target domain hidden representations embedded with a kernel.

Despite the flourishing literature on representation-based domain adaptation methods, they have critical limitations. \citet{zhao2019learning} and \citet{johansson2019support} have presented synthetic examples in which a domain-adversarial model that minimizes the supervised loss in the source domain, while aiming for an invariant representation, still fails in the target domain. We will demonstrate this failure through our experiments in the next section and show that our proposed algorithm remains robust in such situations.


Domain-adaptation methods that do not rely on representation learning are less studied and can be applied in highly restricted settings. For example, importance sampling \citep{shimodaira2000improving} assumes label conditional distribution must be the same and target domain is within the support of the source domain. Our work supplements this direction. 




\section{Experiments}
\label{sec:experiments}

In this section, we first design a synthetic dataset to verify that our regularizer is indeed beneficial in a distribution shift setting by adjusting the classifier and to perform an ablation study on the role of removing implicit regularization. Then, we demonstrate the effectiveness of our method on a well-known benchmark where classic domain-adversarial methods are known to fail \citep{zhao2019learning}. Last, we show our algorithm's practical utility in a cross-lingual sentiment classification task. 

\subsection{Synthetic Data}

We create a distribution shift scenario where the alignment property is present in the labeled data distribution (Figure \ref{fig:synthetic_rotation}), as theoretically discussed in \S~\ref{sec:RGaussian}. For the source domain (\subref{fig:synthetic_source}), the input is sampled from a two-dimensional Gaussian distribution. The distribution is more spread out in the direction of the first principal component (see the black arrows) which corresponds to a larger singular value. In this task, the two classes are separated along this direction as shown in the figure. The resulting vector of all labels is mostly in the direction of the first singular vector of the representation matrix. We rotate the input by $45^{\circ}$ to create the target distribution in (\subref{fig:synthetic_hyperplanes}). 

We then run the proposed algorithm with hyperparameters $k=\tilde{k}=1$ and different values of $\lambda$ and compared it with the $\ell_2$ regularizer and a domain-adversarial baseline DANN \citep{ganin2016domain} with one hidden layer of width 64. Note that the optimal solution should be independent of $\lambda$ (Prop.~\ref{prop-gaussian}), but $\lambda$ may affect the convergence rate. Figures \ref{fig:synthetic_rotation} (\subref{fig:synthetic_hyperplanes}) to (\subref{fig:synthetic_acc}) show the results. Further details are in Appendix B. In Figure \ref{fig:synthetic_rotation} (\subref{fig:synthetic_hyperplanes}) we see that the solution without regularization separates the classes as they are separated in the source domain. The proposed algorithm finds a separating hyperplane that matches how the classes are separated in the target domain.
Finally, (\subref{fig:synthetic_acc}) shows that our regularizer surpasses both the $\ell_2$ regularizer and the domain-adversarial baseline and achieves a near perfect classification in the target domain in this example. The dark green line in this figure uses 2k epochs and its accuracy is sensitive to $\lambda$. Increasing the number of epochs to 20k (green line) reduces this sensitivity, indicating that, as the theory predicted, the final solution is robust to $\lambda$ and the sensitivity is due to slow optimization.

\begin{figure*}[ht!]
\centering
\label{fig:rotated_Gaussian}
\begin{subfigure}{.3\textwidth}
  \centering
  \includegraphics[width=0.9\linewidth]{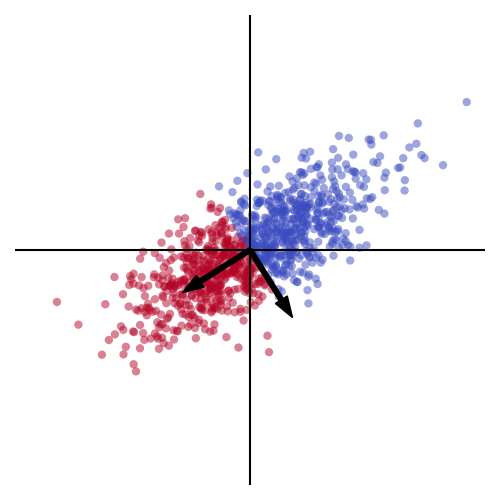}
 \caption{Source Domain}
  \label{fig:synthetic_source}
\end{subfigure}
\begin{subfigure}{.3\textwidth}
  \centering
  \includegraphics[width=0.9\linewidth]{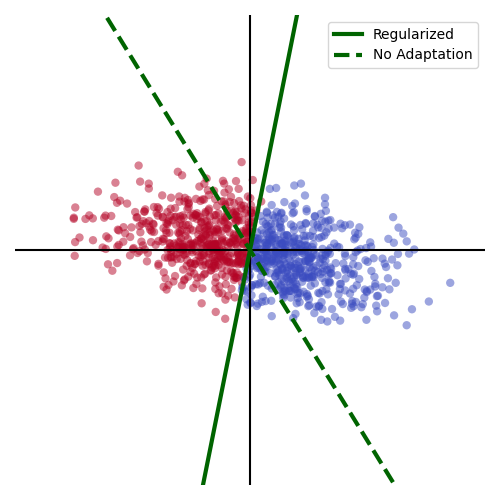}
 \caption{Target Domain}
  \label{fig:synthetic_hyperplanes}
\end{subfigure}
\begin{subfigure}{.35\textwidth}
  \centering
  \vspace{.45cm}
  \includegraphics[width=\linewidth]{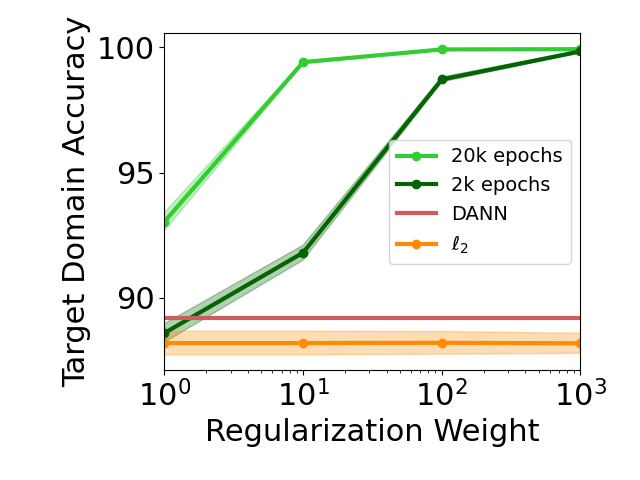}
  \vspace{-.6cm}
 \caption{Performance}
  \label{fig:synthetic_acc}
\end{subfigure}
\small
\caption{(a) Source domain. The black arrows show principal components. (b) Target domain. The green lines show separating hyperplanes found without using any regularization (dashed) and with our regularizer with $\lambda=10^3$ (solid). (c) Performance on the target domain. The red line shows the performance of DANN. The x axis is the regularization coefficient for $\ell_2$ regularization (orange curve) and $\lambda$ for the proposed regularizer (green curves). The proposed regularizer achieves near-perfect accuracy on this domain. Shaded areas are standard errors over 10 runs. Variations in target accuracy of DANN are near zero.}
\label{fig:synthetic_rotation}
\end{figure*}

We also want to evaluate the effectiveness of removing the implicit regularization term $\sum_{i={k+1}}^d \sigma_i^2 (w_i^{\Vv})^2$ as described in Equation \ref{eq:regression_reduction}. More specifically, we are interested in when removing the implicit regularization would be effective.

Recall from Section \ref{sec:theory} that the vanilla closed form- solution without any regularization is:
\begin{equation}
    w = S^{-1}\Phi^{T}y
\end{equation}
where $S = \Phi^{T}\Phi$, and $\Phi \in \Rb^{n\times d}$ is the feature matrix.

To utilize more specific chracteristics of a dataset, we want to first explore the synthetic data case as described in Section \ref{sec:experiments}. Then the closed-form solution with label alignment regularization with removal of the implicit regularization term is:
\begin{equation}
\begin{aligned}
w &= (S_k + \lambda(\Tilde{S} - \Tilde{S}_k))^{-1}\Phi^Ty\\
&= (s_1^2p_1p_1^T + \lambda \Tilde{s}_2^2\Tilde{p}_2\Tilde{p}_2^T)^{-1} \sqrt{\frac{2}{\pi}}s_1p_1
\end{aligned}
\end{equation}
and the closed-form solution with label alignment regularization without implicit regularization removal is:
\begin{equation}
\begin{aligned}
w &= (S + \lambda(\Tilde{S} - \Tilde{S}_k))^{-1}\Phi^Ty\\
&= (s_1^2p_1p_1^T + s_2^2p_2p_2^T + \lambda \Tilde{s}_2^2\Tilde{p}_2\Tilde{p}_2^T)^{-1} \sqrt{\frac{2}{\pi}}s_1p_1.
\end{aligned}
\end{equation}

The only difference is $s_2^2p_2p_2^T$, and therefore the relative magnitude of $s_2$ and $\lambda$ should be the deciding factor of the solution $w$. Experiments conducted on synthetic data corroborated the theoretical conclusion, as illustrated in Figure \ref{fig:synthetic_implicit_removal_ablation}. As seen in the figure, $s_2$ is larger compared to the previous experiment. The target distribution is a rotation of this new distribution. Removing implicit regularization results in a different performance especially with smaller values of $\lambda$.

\begin{figure}[ht]
\begin{subfigure}{.30\textwidth}
  \centering
  \includegraphics[width=\linewidth]{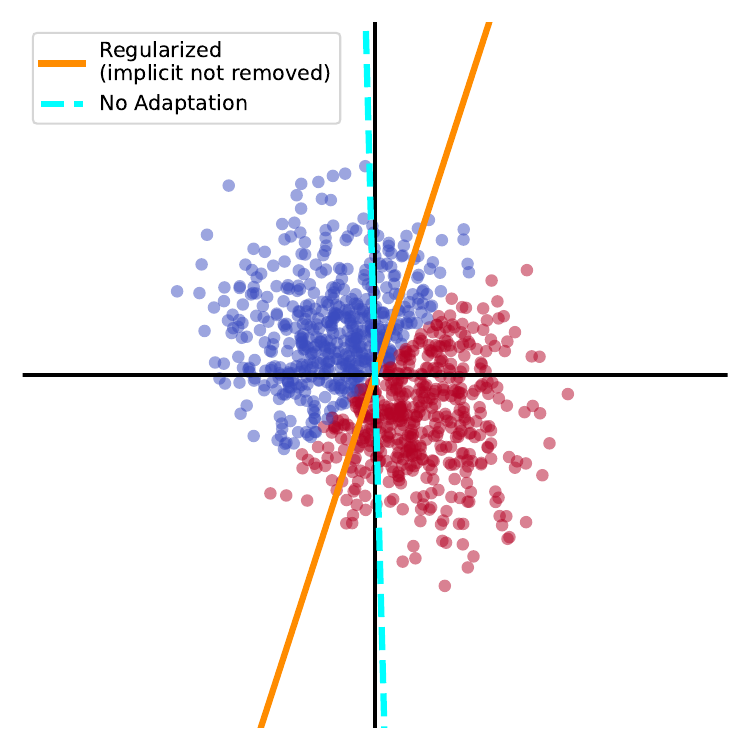}
 \caption{Without Implicit Removal}
  \label{fig:synthetic_implicit_removed}
\end{subfigure}
\begin{subfigure}{.30\textwidth}
  \centering
  \includegraphics[width=\linewidth]{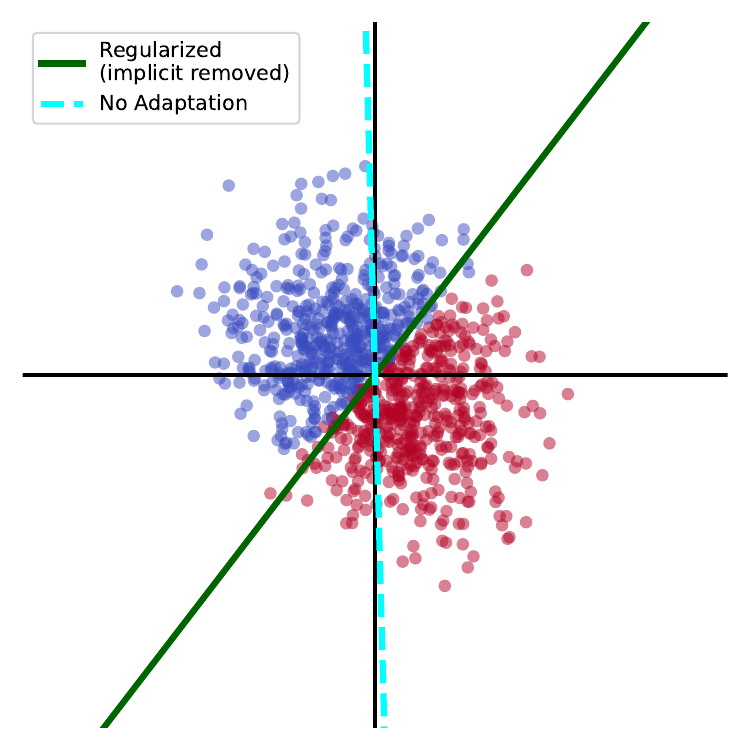}
 \caption{With Implicit Removal}
  \label{fig:synthetic_no_implicit_removed}
\end{subfigure}
\begin{subfigure}{.35\textwidth}
  \centering
  \vspace{.95cm}
  \includegraphics[width=\linewidth]{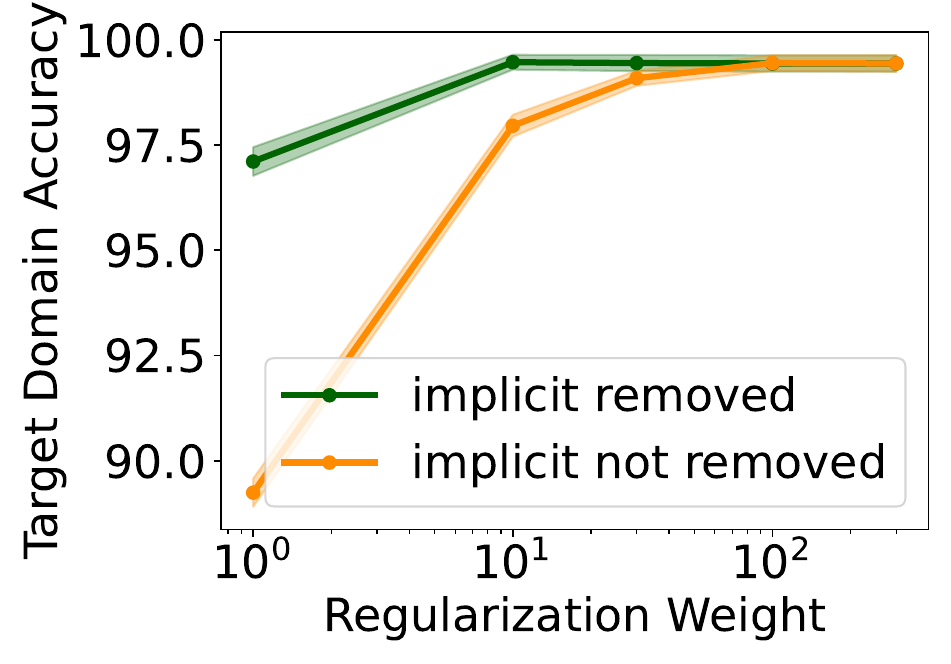}
 \caption{Performance}
  \label{fig:synthetic_performance_implicit_removed}
\end{subfigure}
\caption{(a) Without Implicit Removal. The cyan dashed line is the decision boundary without any adaptation. The orange line shows the decision boundary when $\lambda$ is set to 1 for our proposed regularizer without implicit removal. (b) With Implicit Removal. The green line shows the decision boundary when $\lambda$ is set to 1 with implicit removal. (d) Performance on the target domain. The horizontal axis is $\lambda$ for the proposed regularizer. Before $\lambda$ dominates, the benefits of removing implicit regularization are significant. Shaded areas are standard errors over 10 runs.}
\label{fig:synthetic_implicit_removal_ablation}
\end{figure}

\subsection{MNIST-USPS} 

The experiments in Table \ref{table:m-u} consider binary classification tasks from the MNIST-USPS domain adaptation benchmark with linear and shallow models. Both MNIST and USPS are digit classification datasets with 10 classes and therefore 45 binary classification tasks between two digits. In MNIST, the input is a 28$\times$28 grayscale image flattened to a 784-dimensional vector. USPS images are 16$\times$16 and we resize them to 28$\times$28 and flatten each input to a 784 vector.

Each column of the table is the average accuracy over 45 domain adaptation tasks. In the first column, the source domain (fully labeled) is a pair of digits from MNIST and the target domain (fully unlabeled) is the same pair, but from USPS. The datasets for the source and target domains are reversed in the second column. The last three columns are like the second column except that, in binary classification between two digits, only a certain ratio of the lower digit in the source domain, as indicated in the header, is used. This subsampling creates a large degree of imbalance that, as \citet{zhao2019learning} observed, poses a challenge to domain-adversarial methods.

We use the train split of the dataset for the source domain and the test split of the other dataset for the target domain. A small set of 100 labeled points from the target domain is used for hyperparameter selection as we have not developed a fully unsupervised hyperparameter selection strategy. However, we give the baseline the same validation set to keep the experiment fair. 

The first two rows show the performance of the domain-adversarial method DANN \citep{ganin2016domain} with one hidden layer on these tasks. (Deeper NNs performed worse on the highly imbalanced tasks in our preliminary experiments.) The first row is the average target domain accuracy of a two-layer ReLU NN trained purely on the source domain. In the second row, the domain-adversarial objective is added to reduce domain shift in the hidden representation. DANN improves accuracy in both U$\to$M and M$\to$U. In the cases with subsampling, however, DANN consistently hurts performance. The third and fourth row show the performance of a linear method with or without our regularizer. Using our regularizer improves performance in all columns and outperforms the models in the other rows in the cases with subsampling.

\begin{table*}[ht!]
\centering

\begin{tabular}{llllll}
\hline
            &   U $\rightarrow$ M &   M $\rightarrow$ U &   0.3 $\rightarrow$ U &   0.2 $\rightarrow$ U &   0.1 $\rightarrow$ U \\
\hline
 No Adaptation (NN) &           77.85 &           84.88 &                \textbf{83.36} &                \textbf{72.84} &                \textbf{53.58} \\
 DANN          &           \textbf{83.93} &           \textbf{86.69} &                78.05 &                64.2  &                47.27 \\
\hline
 No Adaptation (Linear) &           78.68 &           83.84 &                80.99 &                79.47 &                75.41 \\
 Label Alignment Regression          &           \textbf{81.97} &           \textbf{88.96} &                \textbf{86.99} &                \textbf{84.84} &                \textbf{82.71} \\
\hline
\end{tabular}
\small 
\caption{Accuracies on MNIST-USPS benchmark. Each column is averaged over the 45 binary classification tasks. M and U indicate MNIST and USPS. Ratios indicate MNIST tasks where one digit is subsampled. In tasks with severe subsampling the proposed algorithm improves the accuracy and achieves the highest performance. DANN performs worse than a regular neural network under subsampling.}
\label{table:m-u}
\end{table*}

We then investigate why DANN hurts performance under subsampling. A domain-adversarial network like DANN has three components: a domain classifier (discriminator) that predicts whether a data point is from the source or the target domain, a generator that learns a shared embedding between the two domains, and a label predictor that performs classification on the task of interest using the generator's embedding. The label predictor uses the labeled source data to increase source accuracy, i.e. the label predictor's accuracy on the source domain. The ultimate goal is to have the label predictor achieve high accuracy on the target domain. The discriminator's accuracy on the other hand shows how successful the discriminator is in recognizing whether a point is from the source or the target domain. In an ideal case this accuracy should be close to that of a random classifier since the data points from the two domains are mapped close to each other in the shared embedding.

Table \ref{table:m-u-dann} shows the average source domain accuracy and domain classifier accuracy of DANN. Average source accuracy remains $\ge95\%$ and average domain classifier remains $\approx 50\%$, indicating that DANN has managed to learn a representation that is suitable for the source domain and maps the points from the source and target domain close to each other. The large drop in DANN's performance can be attributed to the fact that the representation maps positive points in the source domain close to negative points in the target domain and vice-versa and therefore the joint error of the best hypothesis on the two domains (as described in Section \ref{sec:related}) is large. We verify this by training a nearest neighbour (1-NN) classifier on the learned representation in the subsampled settings. The 1-NN classifier uses the source domain representations as the training data and the target domain representations as the test data. The accuracy of this classifier will suffer if in the learned representation the source domain points from one class are mapped close to the target domain points from the other class. The third row in the table, which is also averaged over the 45 tasks, shows a noticeable drop in the performance of the 1-NN classifier and indicates that this problem is present in the learned embeddings.

\begin{table*}[ht!]
\centering

\begin{tabular}{llllll}
\hline
            &   U $\rightarrow$ M &   M $\rightarrow$ U &   0.3 $\rightarrow$ U &   0.2 $\rightarrow$ U &   0.1 $\rightarrow$ U \\
\hline
 Source Accuracy &           98.06 &           98.83 &                98.3  &                97.56 &                95.3 \\
 Discriminator Accuracy          &            46.4 &           50.63 &                50.42 &                50.48 &                50.44 \\
 1-NN Accuracy          &            - &           - &                77.89 &                73.22 &                69.75 \\
\hline
\end{tabular}
\small 
\caption{Source accuracy and domain classifier accuracy of DANN on MNIST-USPS. The drop in source accuracy under severe subsampling is minimal compared to the drop in target accuracy in the previous table. The domain classifier accuracy is near random regardless of the amount of subsampling. The performance of a nearest neighbour classifier trained on the mapped source data points and evaluated on the mapped target data points degrades to a large extent with more subsampling.}
\label{table:m-u-dann}
\end{table*}

\subsection{Cross-Lingual Sentiment Classification}

This section includes cross-lingual sentiment analysis experiments on deep features. XED \citep{ohman-etal-2020-xed} is a sentence-level sentiment analysis dataset consisting of 32 languages. We use English as the source domain and another language as the target domain and create 9 binary classification domain adaptation tasks. 

There are a total of 1984 language pairs from each of the 32 languages to another. We chose 9 language pairs before running the experiment. The sentences in the dataset are labeled with one or more emotions \textit{anger}, \textit{anticipation}, \textit{disgust}, \textit{fear}, \textit{joy}, \textit{sadness}, \textit{surprise}, and \textit{trust}. Following the authors' guidelines we turn these multi-label classification tasks to binary classification by labeling data points positive if their original labels only include \textit{anticipation}, \textit{joy}, and \textit{trust}, and negative if the original labels only include \textit{anger}, \textit{disgust}, \textit{fear}, and \textit{sadness}. (\textit{Surprise} is discarded.) 

We perform 5 runs and in each one 100 points are randomly sampled from the target domain for validation and the rest are used for evaluation. Similar to the previous experiment, this validation set is used for all algorithms with hyperparameter configurations discussed in Appendix B to have a fair comparison. The representations for the source and target domain are 768-dimensional sentence embeddings obtained with BERT \citep{devlin-etal-2019-bert} models pre-trained on the corresponding languages. The experiment compares Label Alignment Regression with the following baselines.

\paragraph{Source:} This baseline trains a linear regression model with squared error (MSE) or a logistic regression model with crossentropy loss (CE) directly on the source domain and evaluates it on the target domain.

\paragraph{Adversarial-Refine \citep{conneau2017word}:} This baseline uses a domain-adversarial approach to learn a linear transformation that maps the source and target domain into a shared space. A refinement step then encourages the transformation to be orthogonal. This approach has shown promising results in several cross-lingual NLP tasks with word embeddings. We train a linear regression model with squared error (MSE) or a logistic regression model with crossentropy loss (CE) on the source data using the learned shared space and then evaluate it on the target domain.

\paragraph{CDAN \citep{long2017deep}:} Recall that a domain-adversarial method consists of a domain classifier (discriminator) that predicts whether a data point is from the source or the target domain, a generator that learns a shared embedding between the two domains, and a label predictor that performs classification on the task of interest using the generator's embedding. CDAN improves on DANN by conditioning the domain classifier on the label predictor's prediction. The motivation is the improvements observed by incorporating this modification to generative adversarial networks \citep{goodfellow2020generative,mirza2014conditional}.

\paragraph{$f$-DAL \citep{acuna2021f}:} This approach modifies DANN to use a separate domain classifier for each class which allows minimizing a family of divergence measures between the source and domain embeddings. We use $f$-DAL to minimize Pearson $\chi^2$ divergence as the authors had observed superior performance with this divergence in previous vision and NLP benchmarks.

\paragraph{IWDAN and IWCDAN \citep{tachet2020domain}:} These two methods modify DANN and CDAN by incorporating importance sampling to reduce deterioration in performance due to class imbalance. Computing the importance sampling ratios requires access to target domain labels. The authors propose to estimate the ratios and provide the theoretical requirements for the estimation to be accurate.

Table \ref{table:xed} shows $F_1$ scores for the nine tasks and the average score over the tasks. The first 8 rows are the baselines above and in the last row (LAR) we employ the Label Alignment Regression algorithm. The proposed algorithm achieves the highest $F_1$ score on seven out of the nine tasks as well as on average over the tasks. Adv - Refine, CDAN, and $f$-DAL do not provide a consistent benefit over No Adaptation. The two methods with importance weightings, IWDAN and IWCDAN, find better solutions than No Adaptation as well as the other domain-adversarial methods, suggesting that the reweighting in this algorithm, even if it is an estimate of the true importance weighting, is beneficial.


\begin{table*}[ht!]
\small

\centering

\begin{tabular}{llllll}
\hline
               & en $\rightarrow$ bg        & en $\rightarrow$ br        & en $\rightarrow$ cn        & en $\rightarrow$ da        & en $\rightarrow$ de        \\
\hline
 Source (MSE) & 55.22 (0.23) & 53.51 (0.53) & 4.48 (0.27)  & 64.75 (0.14) & 47.95 (0.52) \\
 Source (CE) & 51.55 (0.12) & 56.94 (0.04) & 0.37 (0.00)  & 64.00 (0.18) & 46.55 (0.22) \\
 Adv-R (MSE)      & 46.88 (0.61) & 46.20 (1.56) & 53.54 (1.74) & 51.98 (1.87) & 50.43 (1.07) \\
 Adv-R (CE)      & 45.12 (0.82) & 36.96 (0.95) & 49.87 (1.37) & 50.99 (1.31) & 43.62 (0.66) \\
 CDAN     & 49.99 (5.43) & 31.97 (8.43) & 21.93 (12.76) & 55.80 (5.12) & 33.52 (9.60) \\
 $f$-DAL     & 51.95 (0.88) & \textbf{57.79 (0.50)} & 15.04 (3.65)  & 64.23 (0.14) & 45.74 (0.82) \\
 IWDAN     & 56.16 (1.05) & 55.95 (0.50) & 46.78 (3.55) & 63.41 (1.20) & 46.30 (5.69) \\
 IWCDAN     & 57.70 (1.32) & 54.96 (1.44) & 42.08 (5.85) & 63.64 (1.49) & 41.86 (6.48) \\
 LAR          & \textbf{59.85 (0.08)} & 53.42 (0.33) & \textbf{65.10 (0.24)} & \textbf{65.58 (0.12)} & \textbf{60.46 (0.05)} \\
\hline
\end{tabular}

\begin{tabular}{llllll}
\hline
               & en $\rightarrow$ es        & en $\rightarrow$ fr        & en $\rightarrow$ he        & en $\rightarrow$ hu      & Average  \\
\hline
 Source (MSE) & 39.17 (0.93) & 49.89 (0.57) & 58.19 (0.23) & 59.66 (0.12) &   48.09 (0.37) \\
 Source (CE) & 47.09 (0.20) & 40.90 (0.36) & 58.23 (0.10) & 55.82 (0.06) & 46.83 (0.10) \\
 Adv-R (MSE)     & \textbf{48.37 (1.38)} & 46.45 (0.88) & 48.93 (1.03) & 47.15 (1.18) &   48.88 (0.69)  \\
 Adv-R (CE)     & 41.30 (1.40) & 44.18 (2.16) & 46.95 (1.39) & 44.06 (1.44) & 44.78 (0.40)  \\
 CDAN    & 21.26 (6.97)  & 36.30 (14.82) & 41.29 (9.49) & 34.95 (6.14) & 36.33 (3.73) \\
 $f$-DAL     & 48.25 (12.06) & \textbf{58.23 (0.62)}  & 60.10 (0.31) & 47.42 (1.28) & 49.86 (1.53) \\
 IWDAN     & 36.21 (2.93) & 54.04 (1.72) & 58.18 (0.97) & 56.00 (1.54) & 52.56 (0.84) \\
 IWCDAN     & 37.63 (2.45) & 52.14 (2.39) & 61.08 (0.56) & 55.82 (1.54) & 51.88 (1.24) \\
 LAR          & 43.47 (0.90) & 58.11 (0.24) & \textbf{61.24 (0.09)} & \textbf{59.68 (0.12)} &   \textbf{58.55 (0.17)} \\
\hline
\end{tabular}

\small 
\caption{$F_1$ score in percents on different XED source-language pairs. The numbers in parentheses are standard errors. Adv-R refers to Adversarial-Refine. MSE and CE denote Mean Squared Error and Cross Entropy loss. LAR (Label Alignment Regression) outperforms the baselines on average and on most of the tasks. For adversarial baselines we verified that the discriminator accuracy is near random in this experiment similar to the MNIST-USPS experiment.} \label{table:xed}
\vspace{-0.5cm}
\end{table*}




\section{Discussion and Limitations}

In this work, we proposed a domain adaptation regularization method based on the observation of label alignment property---the label vector of a dataset usually lies in the top left singular vectors of the feature matrix. We show that a regression algorithm in a standard supervised learning task actually contains an implicit regularization method to enforce such a property. Then we demonstrate how we can adapt such a regularization method in a domain adaptation setting. A critical difference between our algorithm and the conventional domain adaptation method is that we do not use regularization to adjust the representation learning. We observe that our algorithm does work well under high imbalance, where the conventional representation-based domain adaptation method fails. We also report improvement over baselines on cross-lingual sentiment analysis tasks.

Immediate next steps are providing an unsupervised hyperparameter selection strategy and extension to multi-class classification. The current method uses a validation set for choosing the hyperparameters. This validation set is remarkably small and on the NLP tasks we found little benefit from involving this set to train a semi-supervised method.

A better hyperparameter selection strategy can also help with applying the proposed method to multi-class classification problems. In Appendix \ref{appdx:additional} we discuss how the regularizer can be extended to multi-class problems using multiple outputs and one-hot labels. In general, the multi-class version would require tuning the hyperparameters separately for each output and the current grid search method would become expensive with large number of classes or fine grids. Using a fixed set of hyperparameter values for all the outputs, we show promising results on the MNIST-USPS benchmark in the same section and leave further exploration to future work.

Other future directions are to investigate the combination of our method and the conventional representation-based domain adaptation method, with the hope that the hybrid method has the advantage of both---it can provide a significant advantage in a broad range of domain-shift settings. It would also be interesting to have a more rigorous theoretical characterization regarding when the label alignment property holds and to what extent the label vector can align with the top singular vectors.


\acks{The authors would like to thank the members of Huawei Noah's Ark Lab, the Reinforcement Learning and Artificial Intelligence lab at the University of Alberta, and Torr Vision Group at the University of Oxford for helpful discussions. This research was funded and supported by Huawei, National Sciences and Engineering Research Council of Canada (NSERC), Canada CIFAR AI Chairs program, and Digital Research Alliance of Canada.}


\newpage

\appendix

\textbf{Table of Contents}
\begin{enumerate}
    \item \Cref{app:labelalign} provides theoretical characterization of label alignment property. 

    \item \Cref{app:proofs-solutions} provides theoretical proofs for relevant theorems from~\Cref{sec:theory}. 

    \item \Cref{sec:app-expdetails} provides experimental details for experiments in the main body of the paper.

    \item \Cref{appdx:additional} provides additional experiments in multiclass classification, parameter sensitivity, and regression.
\end{enumerate}

\section{Label Alignment Property}\label{app:labelalign}
In the proposition below we show that label alignment emerges if multiple features are highly correlated with the labels. The following lemma is needed for the proof.
\begin{lemma} If there are $k' < d$ orthonormal vectors $\{\nu_1, \cdots, \nu_{k'}\}$ such that $\norm{\repmat \nu_i} < \epsilon$ for all $i \in [k']$ then $\repmat_{n\times d}$ has at most $d-k'$ singular values greater than or equal to $\sqrt{k'}\epsilon$. 

\end{lemma}

\begin{proof}
Suppose $\sigma_1, \cdots, \sigma_d$ are the singular values of $\repmat$ sorted in descending order.
The matrix $N_{d\times k'}$ with orthonormal columns that minimizes $\norm{\repmat N}_2$ is the matrix of the last $k'$ right singular vectors, and $\norm{\repmat N}_2 = \sqrt{\sum_{i=d-k'+1}^d \sigma_i^2} \geq \sigma_{d-k'+1}$ (This easily follows from Section 12.1.2 by \citet{bishop2006pattern}). If $\sigma_{d-k'+1} \geq \sqrt{k'}\epsilon$ then for any $N$ with orthonormal columns we have $\norm{\repmat N}_2 \geq \sqrt{k'}\epsilon \implies \norm{\repmat N}_\infty \geq \epsilon$ which contradicts the assumption.
\end{proof}
\begin{proposition}
Suppose $\norm{\labelvec} =1$ and that columns of $\repmat$ are normalized. If $\repmat_{n\times d}$ has $\hat{k} \leq d $ columns $\{\phi_{1}, \cdots, \phi_{ \hat{k}}\}$ where $|\phi_{i}^\top \labelvec| > 1 - \delta$ for all $i \in [\hat{k}]$ and
\begin{itemize}
    \item $0 < \delta < 0.2$
    \item $\hat{k} > 16\delta^2/(-15\delta^2-2\delta+1)$
    \item $d>16\delta^2(\hat{k}-1)$
\end{itemize}  then the norm of the projection of $y$ on the span of the first $k=d-\hat{k}+1$ left singular vectors of $\repmat$ is greater than
\begin{align*}
    \sqrt{\frac{\hat{k}(1-\delta)^2 - 16\delta^2(\hat{k}-1)}{d - 16\delta^2(\hat{k}-1)}}.
\end{align*}
\end{proposition}





\begin{proof} First suppose the dot products in the statement are positive.

Note that for all $i \in \hat{k}$ we have $\norm{\phi_{i} - \labelvec}_2^2 = (\phi_{i} - \labelvec)^\top (\phi_{i} - \labelvec) = \phi_{i}^\top \phi_{i} + \labelvec^\top \labelvec - 2 \phi_{i}^\top \labelvec = 2 - 2\phi_{i}^\top \labelvec < 2\delta$. Due to triangle inequality, $\norm{\phi_{i} - \phi_{j}}_2^2 \leq \norm{\phi_{i} - y}_2^2 + \norm{\phi_{j} - y}_2^2 < 4\delta$.

The span of $\phi_{\hat{k}}, \phi_{\hat{k}+1},\cdots\phi_{d}$ has at most $d-\hat{k}+1$ dimensions. Choose $\hat{k}-1$ orthonormal vectors $\nu_1,\cdots,\nu_{\hat{k}-1} \in \Rb^d$ that are perpendicular to this subspace. Then for any $i,j \in [\hat{k}-1]$ we have $\phi_{i}^\top \nu_j = (\phi_{i} - \phi_{\hat{k}} + \phi_{\hat{k}})^\top \nu_j = (\phi_{i} - \phi_{ \hat{k}})^\top \nu_j + 0 \leq \norm{\phi_{ i} - \phi_{ \hat{k}}} < 4\delta$. Therefore $\norm{\repmat\nu_j} < 4\delta \sqrt{\hat{k}-1}$. Putting this orthonormal basis in the lemma above gives that $\repmat$ has at most $d-\hat{k}+1$ singular values greater than or equal to $4\delta(\hat{k}-1)$.

Now see that $\norm{\repmat^\top \labelvec}^2 = \norm{\sum_{i=1}^{d}\phi_{i}^\top \labelvec}^2$ and is also equal to $\sum_{i=1}^{d}(\s_i \labelvec^{\Umat}_i)^2$. Therefore $\sum_{i=1}^{d}(\s_i \labelvec^{\Umat}_i)^2 \geq  \norm{\sum_{i=1}^{\hat{k}}\phi_{i}^\top \labelvec}^2 > \hat{k}(1-\delta)^2$. Since the columns are normalized, $\sum_{i=1}^d \sigma_i^2 = \norm{\repmat}_F = d$. In addition, we have shown that the last $\hat{k}-1$ singular values are smaller than $4\delta\sqrt{\hat{k}-1}$. Define $\hat{\labelvec}$ as the projection of $\labelvec$ on the first $d-\hat{k}+1$ singular vectors of $\repmat$. Then we have $\hat{k}(1-\delta)^2 < \sum_{i=1}^{d}(\s_i \labelvec^{\Umat}_i)^2 = \sum_{i=1}^{d-\hat{k}+1}(\s_i \labelvec^{\Umat}_i)^2 + \sum_{i=d-\hat{k}+2}^{d}(\s_i \labelvec^{\Umat}_i)^2 < d\norm{\hat{y}}^2 + 16\delta^2(\hat{k}-1)(1-\norm{\hat{y}}^2)$. Rearranging the terms (with the extra conditions in the proposition statement) gives
\begin{align*}
    \norm{\hat{y}} > \sqrt{\frac{\hat{k}(1-\delta)^2 - 16\delta^2(\hat{k}-1)}{d - 16\delta^2(\hat{k}-1)}}
\end{align*}

The inequality is tight in the extreme case where $\hat{k} = d$ and $\delta\to 0$ which results in the label vector being fully in the direction of the first left singular vector and all the other singular values tending to zero.

Now suppose some of the dot products in the statement are negative. We can multiply those columns with $-1$ and prove the result above for this modified matrix. The result holds for the original matrix since this operation only changes the right singular vectors of $\repmat$ and does not affect the left singular vectors or the singular values.
\end{proof}
Let us demonstrate the emergence of alignment and the behavior of the bound when multiple features are highly correlated with the output. In this toy experiment the label vector is sampled from a 1000-dimensional Gaussian distribution $\Nc(\zeros, \eye)$ with mean zero and standard deviation 1 and then normalized to norm one. The matrix $\repmat$ has 10 columns. The first 9 columns are sampled from $\Nc(y, s^2\eye)$ with mean $y$ and a small standard deviation $s$ and the other column is sampled from $\Nc(\zeros, \eye)$. All columns are then normalized to norm one. Note that the proposition above does not assume Gaussian features. Figure \ref{fig:theory} shows the norm of the projection of the label vector on the first two singular vectors at different levels of $s$ and its relationship with $\delta$.

\begin{figure*}[t]
\begin{subfigure}{.49\textwidth}
  \centering
  \includegraphics[width=.9\linewidth]{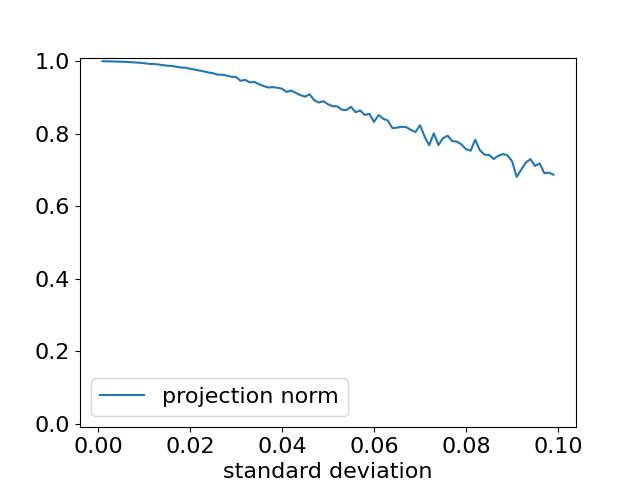}
\end{subfigure}
\begin{subfigure}{.49\textwidth}
  \centering
  \includegraphics[width=.9\linewidth]{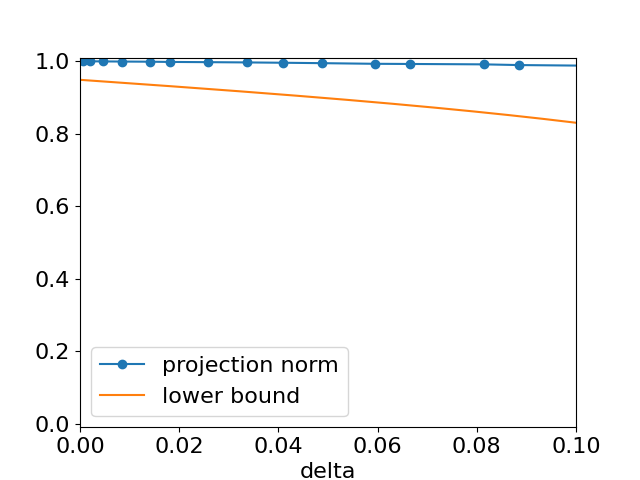}
\end{subfigure}
\caption{Projection of the label vector on the top two singular vectors in the Gaussian example. For small values of standard deviation (where the labels are highly correlated with the features) and small values of $\delta$, the label vector is mostly in the direction of the top two singular vectors. The lower bound is applicable in this regime and is close to one.}
\label{fig:theory}
\end{figure*}

\section{Proofs in Section~\ref{sec:theory}}\label{app:proofs-solutions}

\newcommand{\tglastk}{\tilde{\Sigma}_{-\tilde{k}}}
\newcommand{\sctopk}{\Sigma_k}
\newcommand{\tgtopk}{\tilde{\Sigma}_{\tilde{k}}}

\newcommand{\tgV}{\tilde{V}}
\newcommand{\tgU}{\tilde{U}}
\newcommand{\tgS}{\tilde{\Sigma}}


\begin{lemma}\label{lem:expect-xy}
In the rotated Gaussian example in Section \ref{sec:theory}, $\Eb_{x, y} [x y] = \sqrt{\frac{2}{\pi}} s_1 p_1.$
\end{lemma}
\begin{proof}
\begin{align}
\frac{1}{n}\Phi^\top y &\approx \Eb_{x, y} [x y] \\
& = \int_x \int_y x y p_{\Sc}(x|y)p(y) dy dx \\
& = \int_x x p_{\Sc}(x|y=1)p(y=1)- x p_{\Sc}(x|y=-1)p(y=-1) dx \\
& = \int_x x \cdot 2 \Nc(0, Q) \big(\mathds{1}(x_1^P > 0)p(y=1)- \mathds{1}(x_1^P < 0)p(y=-1)\big) dx, \label{eq:lemma9_1}
\end{align}
where we plug into the definition~\eqref{eq:def-pxgiveny} to get the last equality. Further note that $\mathds{1}(x_1^P < 0)=1-\mathds{1}(x_1^P > 0)$ and  plug this into above,
\begin{align}
\eqref{eq:lemma9_1} & = \int_x x \cdot 2 \Nc(0, Q) \big(\mathds{1}(x_1^P > 0)-p(y=-1)\big) dx \\
& = \int_x x \cdot 2 \Nc(0, Q)\mathds{1}(x_1^P > 0)dx \\
& = 2P^\top \int_z z \cdot \frac{1}{2\pi s_1 s_2}\exp\left( -\frac{1}{2 s_1^2} z_1^2 -\frac{1}{2 s_2^2}z_2^2 \right) \mathds{1}(z_1 > 0)dz
\end{align}
where in the last equality we let $z=Px$, then $x^P_1 = x^\top p_1 = z^\top P p_1 = z_1$. The integral above is a vector with two elements because it includes $z$. The first element is
\begin{align}
& \int_{z_1} \int_{z_2} z_1 \cdot \frac{1}{2\pi s_1 s_2}\exp\left( -\frac{1}{2 s_1^2} z_1^2 -\frac{1}{2 s_2^2}z_2^2 \right) \mathds{1}(z_1 > 0)dz_1 dz_2 \\
&= \int_{z_1} z_1 \cdot  \frac{1}{\sqrt{2\pi}s_1}\exp\left( -\frac{1}{2 s_1^2} z_1^2\right) \mathds{1}(z_1 > 0)dz_1 \\
&= \frac{1}{2} \int_{0}^{+\infty} z_1 \cdot  \frac{\sqrt{2}}{\sqrt{\pi}s_1}\exp\left( -\frac{1}{2 s_1^2} z_1^2\right) dz_1 \\
&=\frac{1}{2} s_1\sqrt{\frac{2}{\pi}} 
\end{align}
The last equality is because the integration is the mean of half-normal distribution. The second element would become zero as written below and noting that $\Eb[z_2]$ is the mean of a zero-mean Gaussian random variable:
\begin{align}
& \int_{z_1} \int_{z_2} z_2 \cdot \frac{1}{2\pi s_1 s_2}\exp\left( -\frac{1}{2 s_1^2} z_1^2 -\frac{1}{2 s_2^2}z_2^2 \right) \mathds{1}(z_1 > 0)dz_1 dz_2 \\
&= \int_{z_1} \frac{1}{\sqrt{2\pi}s_1}\exp\left( -\frac{1}{2 s_1^2} z_1^2\right) \mathds{1}(z_1 > 0)dz_1 \Eb[z_2] = 0 
\end{align}
Then 
\begin{equation}
\Eb_{x, y} [x y] = 2P^\top \begin{bmatrix} \frac{1}{2} s_1 \sqrt{\frac{2}{\pi}} \\
0 \end{bmatrix} = \sqrt{\frac{2}{\pi}} s_1 p_1.
\end{equation}
\end{proof}

\rGauss*
\begin{proof}
We rewrite \eqref{eq:label_alignment} as:
\begin{align}
(s_1^2 v_1 v_1^\top + \lambda\tilde{s}_2^2 \tilde{v}_2 \tilde{v}_2^\top)\widehat{w^*} = \sqrt{\frac{2}{\pi}} {s_1} v_1.
\end{align}
Suppose $\widehat{w^*} = w_1 \tv_1 + w_2 \tv_2$ with $w_1 \in \Rb, w_2 \in \Rb$, then the equation above becomes:
\begin{align}\label{eq:find_k1}
s_1^2 v_1 (v_1^\top \widehat{w^*}) + \lambda\tilde{s}_2^2 w_2 \tilde{v}_2 = \sqrt{\frac{2}{\pi}} {s_1} v_1.
\end{align}
Apply $v_2^\top$ on both sides we have:
\begin{align}
\lambda\tilde{s}_2^2 w_2 \tilde{v}_2^\top v_2 = 0.
\end{align}
Since $\tilde{v}_2$ is not parallel to $v_1$, we must have $\tilde{v}_2^\top v_2 \neq 0$ and thus $w_2 = 0$. To obtain the exactly value of $w_1$, solve \eqref{eq:find_k1} by setting $w_2=0$, we have:
$
s_1^2 (v_1^\top \tv_1) w_1 = \sqrt{\frac{2}{\pi}} {s_1}, \, w_1 = \sqrt{\frac{2}{\pi}} \frac{1}{s_1 v_1^\top \tv_1}.
$
\end{proof}

\alignment*

\begin{proof}
From the definition of $\widehat{w^*}$, we see that:
\begin{align}
(\sum_{i\leq k} \sigma_i^2 v_i v_i^\top + \lambda \sum_{j > \tk} \tsigma_j^2 \tv_j \tv_j^\top ) \widehat{w^*}= \sum_{i\leq k} \sigma_i y_i^U v_i.
\end{align}
Decompose $\ww = \sum_{i\leq d} w_i \tv_i$. From the equation above we find:
\begin{align}\label{eq:sol_wstar}
\sum_{i\leq k} \sigma_i^2 v_i v_i^\top \ww + \lambda \sum_{j > \tk} \tsigma_j^2 \tv_j w_j = \sum_{i\leq k} \sigma_i y_i^U v_i.
\end{align}
Applying $v_m^\top$ only both sides with $m > k$, we have:
\begin{align}
\sum_{j > \tk} \tsigma_j^2 v_m^\top \tv_j w_j = 0, \, m > k,
\end{align}
which can be written as:
\begin{align}
(V'_{d-k})^\top \tilde{V'_{d-k}} \, {\rm diag}(\tsigma_{\tk + 1}^2, \dots, \tsigma_d^2) \begin{bmatrix}
w_{\tk + 1} \\
\dots \\
w_d
\end{bmatrix} = \zero.
\end{align}
Note that multiplying by ${\rm diag}(\tsigma_{\tk + 1}^2, \dots, \tsigma_d^2)$ does not change the invertibility. By assumption we must have $w_{k+1} = \dots = w_d = 0$, and \eqref{eq:sol_wstar} becomes independent of $\lambda$.
\end{proof}


\optimal*

\begin{proof}
With the assumption of Theorem \ref{thm:alignment}, we have: $v_i^\top \ww = y_i^U / \sigma_i$, for $i\leq k$. Then we can write down the optimal solutions as $w_\Sc^* = V_k \mu_k$, $w_\Tc^* = \tV_\tk \tilde{\mu}_{\tk}$. Since $V_k^\top \tV_k$ is invertible (and under the assumptions of Theorem \ref{thm:alignment}), then we obtain the label alignment regularized result 
\begin{align}\label{eq:sol_label_align}
\ww = \tV_k (V_k^\top \tV_k)^{-1} \mu_k.
\end{align}
Note that the solution is independent of the hyperparameter $\lambda$. 

From $\ww = \tV_k (V_k^\top \tV_k)^{-1} \mu_k$, and $w_\Tc^* = \tV_\tk \tilde{\mu}_{\tk}$, the equation $\ww = c w_\Tc^*$ holds iff:
\begin{align}
\tV_k((V_k^\top \tV_k)^{-1} \mu_k  - c \tilde{\mu}_k) = \zero,
\end{align}
or in other words, $(V_k^\top \tV_k)^{-1} \mu_k  =  c\tilde{\mu}_k + q$, where $q \in {\rm null}(\tV_k) = \{\zero\}$. 
\end{proof}


\invertible*

\begin{proof}
$V_k^\top \tV_k$ can be written as $[v_i^\top \tv_j]$ with $i, j \in [k]$. From the assumption, $V_k^\top \tV_k = \eye + \epsilon \Delta$ where $\eye$ is the identity matrix and $\Delta$ is a $k\times k$ matrix with each element $|\Delta_{ij}| \leq \epsilon$. Suppose $(\eye + \epsilon \Delta) x = \zero$, then $x = -\epsilon \Delta x$, taking the norm on both sides we have:
\begin{align}
\|x\| &= \epsilon \| \Delta x\| \leq \epsilon \|\Delta \| \cdot \|x\| \\
&\leq \epsilon \|\Delta \|_F \cdot \|x\| \leq \epsilon k \cdot \|x\| < \|x\|,
\end{align}
which gives $x = \mathbf{0}$. Therefore, $\eye + \epsilon \Delta$ is invertible. Similarly, $(V'_{d-k})^\top \tilde{V}_{d-k}'$ is also invertible.
\end{proof}

\invertibleStrong*

\begin{proof}
It suffices to show that
$P(\det(V_k^\top \tV_k) = 0) = 0$. Note that $\det(V_k^\top \tV_k) = 0$ can be rewritten as:
\begin{align}
\det(\tv_1^{V_k} \dots \tv_k^{V_k}) = 0, 
\end{align}
and thus 
\begin{align}
P(\det(V_k^\top \tV_k) = 0) \leq & p(\tv_1^{V_k} = 0) + p(\tv_2^{V_k} \in {\rm span}(\tv_1^{V_k})|\tv_1) +  \dots + \\
& p(\tv_k^{V_k} \in {\rm span}(\tv_1^{V_k}, \dots, \tv_{k-1}^{V_k})|\tv_1, \dots, \tv_{k-1}).
\end{align}
Since each condition gives a sub-manifold with a smaller dimension and the probability distributions are continuous, from Sard's theorem \citep[e.g.][]{guillemin2010differential}, each probability is zero. Therefore, $P(\det(V_k^\top \tV_k) = 0) = 0$. %
\end{proof}

\section{Experiment Details}\label{sec:app-expdetails}

This section outlines dataset details, hyperparameter settings, and other design choices in the experiments.

\paragraph{Label Alignment in Real-World Tasks.} We used the following tasks in this experiment:

\begin{enumerate}
    \item \textbf{UCI CT Scan:} A random subset of the CT Position dataset on UCI \citep{graf20112d}. The task is predicting a location of a CT Slice from histogram features.
    \item \textbf{Song Year:} A random subset of the training portion of the Million Song dataset \citep{Bertin-Mahieux2011}. The task is predicting the release year of a song from audio features.
    \item \textbf{Bike Sharing:} A random subset of the Bike Sharing dataset on UCI \citep{fanaee2014event}. The task is predicting the number of rented bikes in an hour based on information about weather, date, and time.
    \item \textbf{MNIST:} The task is classifying any pair of two digits in MNIST.
    \item \textbf{USPS:} The task is classifying any pair of two digits in USPS.
    \item \textbf{CIFAR-10:} The task is classifying airplane and automobile in CIFAR-10 dataset using features from a ResNet-18 pretrained on ImageNet.
    \item \textbf{CIFAR-100:} The task is classifying beaver and dolphin in CIFAR-100 dataset using features from a ResNet-18 pretrained on ImageNet.
    \item \textbf{STL-10:} The task is classifying airplane and bird in STL-10 dataset using features from a ResNet-18 pretrained on ImageNet.
    \item \textbf{XED (English):} The English corpus from XED datsets whose details are discussed in the main paper. The features are sentence embeddings extracted from BERT.
    \item \textbf{AG News:} A random subset of the first two classes (World and Sports) in AG News document classification dataset. The features are obtained by feeding the document text to BERT.
\end{enumerate}

All datasets have an extra constant 1 feature to account for the bias unit. Rank is computed as the number of singular values larger than $\sigma_1 * \max(n, d) * 1.19209e-07$. This is the default numerical rank computation method in the Numpy package.

\paragraph{MNIST-USPS.} 

DANN uses a one-layer ReLU neural network. This is the Shallow DANN architecture suggested by the original authors \citep{ganin2016domain}. We swept over values of 128, 256, 512, and 1024 for the depth of the hidden layer.

The neural network is trained for 10 epochs using SGD with batch size 32, learning rate 0.01, and momentum 0.9. This model already achieves near perfect accuracy on the source domain.

Candidate hyperparameter values for Label Alignment Regularizer were $\{1e-1, 1e+1, 1e+3\}$ for $\lambda$ and $\{8, 16, 32, \cdots\}$ up to the rank of $\Phi$ or $\tilde{\Phi}$ for $k$ and $\tilde{k}$.

Although the number of hyperparameter configurations is greater for our method, this experiment is in favor of DANN if we take runtime into account.

The linear model is trained using full-batch gradient descent for 5000 epochs with learning rate $1/(2\sigma_1)$.

\paragraph{Sentiment Classification.}

For the domain-adversarial baseline \citep{conneau2017word} we sweep over values of $\{1e-3, 1e-2, 1e-1\}$ for $\beta$. The parameter controls the degree of orhtogonality of the transformation that maps the source and target embeddings into a common space.

The models used in No Adaptation, Adv - Refine, and Label Alignment Regression are linear regression or logistic regression (on the nonlinear extracted representations). These models are trained with learning rate $1/(2\sigma_1)$ (MSE loss) and $1e-2$ (CE loss) and momentum $0.9$.

For CDAN and $f$-DAL we sweep over regularization coefficients $\{1e-4, 1e-2, 1\}$ with a one-hidden-layer ReLU network. This is the architecture suggested by \citep{ganin2016domain} for domain adaptation with a shallow network.

Candidate hyperparameter values for Label Alignment Regularizer were $\{1e-1, 1e+1, 1e+3\}$ for $\lambda$ and $\{8, 16, 32, \cdots\}$ up to the rank of $\Phi$ or $\tilde{\Phi}$ for $k$ and $\tilde{k}$.

Similar to the previous experiment, the hyperparameter grid search in this experiment is in favor of the baselines if we take runtime into account.

\section{Additional Experiments} \label{appdx:additional}

\subsection{Multi-Class Classification}
We also try to generalize the formulation of the label alignment regression to a multiclass setting following the derivation in Equation \ref{eq:lr-rewrite} . Given a dataset comprising $n$ samples, each characterized by $d$ features, we denote the feature matrix by $\Phi \in \mathbb{R}^{n \times d}$. In a classification context involving $c$ distinct classes and employing a one-versus-all strategy, the target matrix $Y$, which adopts a $\pm 1$ style one-hot encoding scheme, is of dimension $\mathbb{R}^{n \times c}$. Consequently, the weight matrix $W$, which maps the feature space to the class labels, is represented as $\mathbb{R}^{d \times c}$. Then the learning objective can be formulated as:
\begin{equation}
\begin{aligned}
\min_W \|\Phi W - Y\|^2 &= \min_W \|U \Sigma V^T W - Y\|^2 \\
&= \min_W \| \Sigma V^T W - U^T y\|^2 \\
&= \min_W \| \Sigma W^V - Y^U\|^2 \\
&= \min_W \sum_{j=1}^{c}\sum_{i=1}^{d} (\sigma_i W^Y_{ij} - Y^U_{ij})^2 + \sum_{j=1}^{c}\sum_{i=d+1}^{n} (Y^U_{ij})^2
\end{aligned}
\end{equation}

The notation $\|A\|^2$ signifies the 2-norm of matrix $A$, encapsulating the square root of the sum of the squares of its elements. In this setup, $W^V$ corresponds to a weight matrix with dimensions $\mathbb{R}^{d \times c}$, while $Y^U$ denotes a modified target matrix also of dimension $\mathbb{R}^{n \times c}$. Thus, the expression $(\Sigma W^V - Y^U)$, representing the discrepancy between the projected feature space and the modified target matrix, retains the dimensionality of $\mathbb{R}^{n \times c}$, highlighting the alignment or misalignment of the model predictions with the modified targets in the given multidimensional space.

Assume $k$ is the same for every one vs all setting and $k<d$, we can obtain
\begin{equation}
\begin{aligned}
\min_W \|\Phi W - Y\|^2 &= \min_W \sum_{j=1}^{c}\sum_{i=1}^{d} (\sigma_i W^Y_{ij} - Y^U_{ij})^2 + \sum_{j=1}^{c}\sum_{i=d+1}^{n} (Y^U_{ij})^2\\
&= \min_W \sum_{j=1}^{c}\sum_{i=1}^{d} (\sigma_i W^V_{ij} - Y^U_{ij})^2\\
&= \min_W \sum_{j=1}^{c}\sum_{i=1}^{k} (\sigma_i W^V_{ij} - Y^U_{ij})^2 + \sum_{j=1}^{c}\sum_{i=k+1}^{d} (\sigma_i W^V_{ij})^2.
\end{aligned}
\end{equation}

Therefore the final objective function looks like
\begin{equation}
\begin{aligned}
\min_W \|\Phi W - Y\|^2 
 - \sum_{j=1}^{c}\sum_{i=k+1}^{d} (\sigma_i W^V_{ij})^2  + \lambda\sum_{j=1}^{c}\sum_{i=\Tilde{k}+1}^{d} (\Tilde{\sigma}_i W^{\Tilde{V}}_{ij})^2.
\end{aligned}
\end{equation}

We also validated the label alignment property by computing $k(0.1)$ for the one-versus-all label vector corresponding to each digit similar to the binary classification case in Table \ref{tab:alignment_tasks}. The value of $k(0.1)$ for all the digits was $1$. 

Then, we compare the classification performance of our label alignment regression to DANN in the multiclass MNIST-USPS classification setting. Our method outperforms DANN by a large margin as shown by the evaluation results in Table \ref{table:m-u-multiclass}.
\begin{table*}[ht!]
\centering

\begin{tabular}{llllll}
\hline
            &   U $\rightarrow$ M &   M $\rightarrow$ U &   0.3 $\rightarrow$ U &   0.2 $\rightarrow$ U &   0.1 $\rightarrow$ U \\
\hline
 No Adaptation (NN) &           35.66&           52.46&                47.24&                45.06&                19.48\\
 DANN          &           41.32&           53.46&                49.88&                43.09&                32.01\\
\hline
 No Adaptation (Linear) &           37.53&           54.41&                48.71&                46.21&                41.54\\
 Label Alignment Regression          &           \textbf{42.47}&           \textbf{63.90}&                \textbf{54.69}&                \textbf{51.70}&                \textbf{47.49}\\
\hline
\end{tabular}
\small 
\caption{Accuracies on MNIST-USPS multiclass benchmark.  M and U indicate MNIST and USPS. Ratios (0.3, 0.2, 0.1) indicate MNIST tasks where 9 out of the 10 digits are subsampled. For the subsampling setting (last three columns), each column is averaged over the 10 subsampling classification tasks. In all tasks, the proposed algorithm improves the accuracy and achieves the highest performance. }
\label{table:m-u-multiclass}
\end{table*}

\subsection{Parameter Sensitivity}
The parameter $\lambda$ indicates the ratio of the loss obtained from the unsupervised information of the target domain.  We want to quantitatively evaluate how this ratio influences the performance of our proposed method on the target domain. The sensitivity visualization is shown in Figure \ref{fig:lambda_sensitivity_curve}. 
\begin{figure}
    \centering
    \includegraphics[width=.5\linewidth]{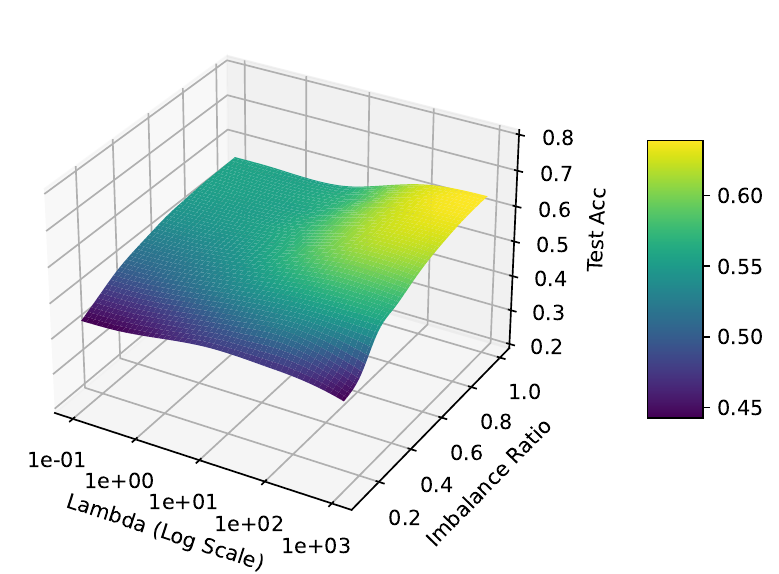}
    \caption{$\lambda$ sensitivity curves of accuracies on MNIST USPS multiclass benchmark.  The performance of the proposed method is relatively invariant over different $\lambda$ under various imbalance (subsampling) ratios. Generally greater $\lambda$ comes with better performance in the target domain because more weight and emphasis of loss is put on the information of the target domain.}
    \label{fig:lambda_sensitivity_curve}
\end{figure}

\subsection{Regression}
We create a regression task similar to the synthetic experiment in the main paper. The aim is to understand if results similar to the synthetic experiment hold in a setting where the labels are not restricted to $\pm1$ and where mean squared error is used for evaluation. All the details are the same as in the classification experiment in the main paper except that the label vector is simply set to the first left singular vector. Figure \ref{fig:synthetic_reg} shows the results and corroborates the findings in the main paper. Note that DANN is not directly applicable here.

\begin{figure}[ht]
\begin{subfigure}{.24\textwidth}
  \centering
  \includegraphics[width=\linewidth]{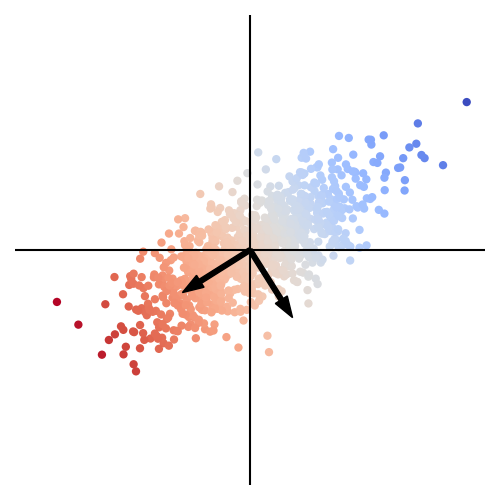}
 \caption{Source domain}
  \label{fig:synthetic_reg_source}
\end{subfigure}
\begin{subfigure}{.24\textwidth}
  \centering
  \includegraphics[width=\linewidth]{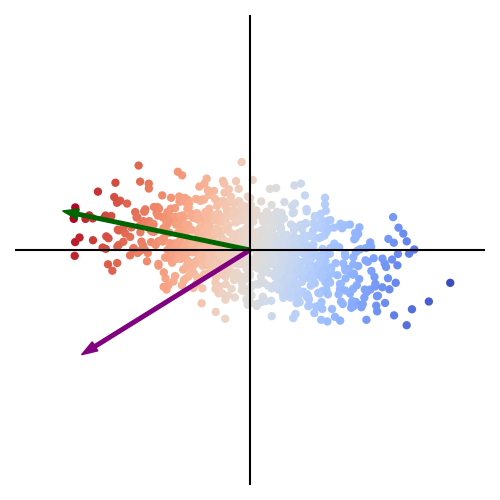}
 \caption{Target Domain}
  \label{fig:synthetic_reg_weights}
\end{subfigure}
\begin{subfigure}{.24\textwidth}
  \centering
  \vspace{.95cm}
  \includegraphics[width=\linewidth]{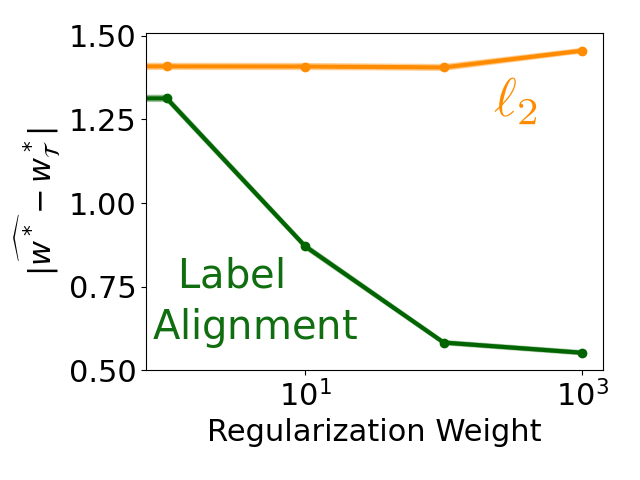}
 \caption{Parameter Error}
  \label{fig:synthetic_reg_wdiff}
\end{subfigure}
\begin{subfigure}{.24\textwidth}
  \centering
  \vspace{.95cm}
  \includegraphics[width=\linewidth]{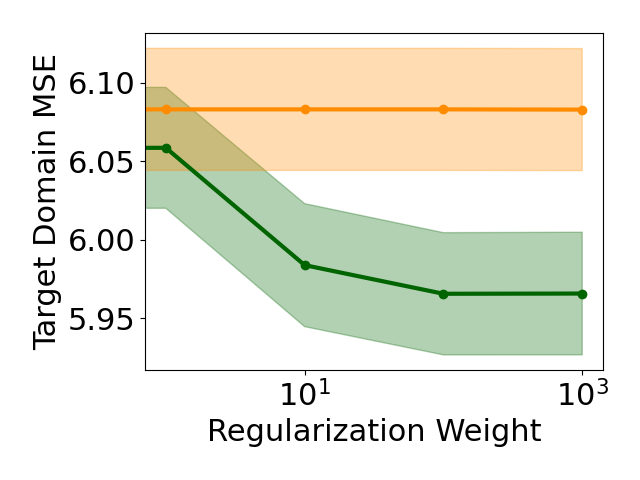}
 \caption{Performance}
  \label{fig:synthetic_reg_mse}
\end{subfigure}

\caption{(a) Source domain. The black arrows show principal components. (b) Target domain. The arrows show weights found without using any regularization (purple) and with our regularizer with $\lambda=10^3$ (green). (c) Distance between the estimated and the optimal weights. The proposed regularizer reduces this distance. (d) Performance on the target domain. The x axis is the regularization coefficient for $\ell_2$ regularization and $\lambda$ for the proposed regularizer. The proposed regularizer achieves lower error on this domain. Shaded areas are standard errors over 10 runs.}
\label{fig:synthetic_reg}
\end{figure}

\vskip 0.2in
\bibliography{sample}

\end{document}